\documentclass[a4]{article}
\usepackage{url}
\usepackage[T1]{fontenc}
\usepackage{ae,aecompl}
\usepackage[utf8x]{inputenc}
\usepackage{amsmath}
\usepackage{empheq}
\usepackage{amsfonts}
\usepackage{amssymb}
\usepackage{amsthm}
\usepackage{graphicx, textcomp, caption}
\usepackage{tikz-cd}
\usepackage[affil-it]{authblk}
\usepackage{enumitem}
\usepackage{placeins}
\usepackage[all]{nowidow}
\definecolor{olivegreen}{rgb}{0.24,0.5,0.19}
\usepackage[pdftex,breaklinks,colorlinks,citecolor=olivegreen]{hyperref}

\let\Oldsection\section
\renewcommand{\section}{\FloatBarrier\Oldsection}

\let\Oldsubsection\subsection
\renewcommand{\subsection}{\FloatBarrier\Oldsubsection}

\let\Oldsubsubsection\subsubsection
\renewcommand{\subsubsection}{\FloatBarrier\Oldsubsubsection}

\title{Singular Sturm-Liouville Problems with Zero Potential ($q=0$) and Singular Slow Feature Analysis}

\author{Stefan~Richthofer\footnote{Electronic address: \texttt{stefan.richthofer@ini.rub.de}; 
Corresponding author}, Laurenz~Wiskott\footnote{Electronic address: 
\texttt{laurenz.wiskott@ini.rub.de}}}
\affil{Institut f\"ur Neuroinformatik,\\ Ruhr-Universit\"at Bochum, Germany}

\providecommand{\abs}[1]{\lvert#1\rvert}

\providecommand{\Bigabs}[1]{\Big\lvert#1\Big\rvert}
\providecommand{\Abs}[1]{\left\lvert #1 \right\rvert}
\providecommand{\norm}[1]{\lVert#1\rVert}

\providecommand{\av}[1]{\left\langle#1\right\rangle}

\providecommand{\id}[0]{\mathbf{I}}
\providecommand{\trph}[0]{\Omega_t}

\DeclareMathOperator*{\coloneq}{\mathrel{\mathop:}=}
\DeclareMathOperator*{\eqcolon}{=\mathrel{\mathop:}}

\DeclareMathOperator*{\eq}{=}

\DeclareMathOperator*{\ra}{\Rightarrow}

\DeclareMathOperator*{\Lra}{\longrightarrow}

\DeclareMathOperator{\pHermite}{He}
\DeclareMathOperator{\Tschebyschow}{T}

\DeclareMathOperator{\orth}{O}

\DeclareMathOperator*{\opmin}{minimize}

\DeclareMathOperator*{\dleq}{\leq}
\DeclareMathOperator*{\dgeq}{\geq}

\providecommand{\optmin}[1]{\displaystyle\opmin_{#1} \qquad}
\DeclareMathOperator*{\subjectto}{subject \; to \qquad}

\providecommand{\vf}[1]{\mathbf{#1}}
\providecommand{\vfc}[2]{#1_{#2}}

\makeatletter
\newcommand{\subalign}[1]{%
	\vcenter{%
		\Let@ \restore@math@cr \default@tag
		\baselineskip\fontdimen10 \scriptfont\tw@
		\advance\baselineskip\fontdimen12 \scriptfont\tw@
		\lineskip\thr@@\fontdimen8 \scriptfont\thr@@
		\lineskiplimit\lineskip
		\ialign{\hfil$\m@th\scriptstyle##$&$\m@th\scriptstyle{}##$\crcr
			#1\crcr
		}%
	}
}
\makeatother

\newtheoremstyle{remarks}
{}
{}
{\small}
{}
{\itshape}
{.}
{ }
{}
\theoremstyle{remarks}
\newtheorem*{rmk}{Remarks}
\newtheorem*{rmk1}{Remark}

\newtheoremstyle{defi}
{}
{}
{\slshape}
{}
{\bfseries}
{}
{\newline}
{}
\theoremstyle{defi}

\newtheorem{definition}{Definition}
\newtheorem{lem}{Lemma}
\newtheorem{thm}{Theorem}

\newtheorem{corr}{Corollary}

\begin{document}

\maketitle

\begin{abstract}
A Sturm-Liouville problem ($\lambda w y = (r y')' + q y$) is \textsl{singular} if its domain is unbounded or if $r$ or $w$ vanish at the boundary. Then it is difficult to tell whether profound results from regular Sturm-Liouville theory apply. Existing criteria are often difficult to apply, e.g.\ because they are formulated in terms of the solution function.

We study the special case that the potential $q$ is zero under Neumann boundary conditions and give simple and explicit criteria, solely in terms of the coefficient functions, to assess whether various properties of the regular case apply. Specifically, these properties are discreteness of the spectrum (BD), self-adjointness, oscillation ($i$th solution has $i$ zeros) and that the $i$th eigenvalue equals the SFA delta value (the total energy) of the $i$th solution.
We further prove that stationary points of each solution strictly interlace with its zeros (in singular or regular case, regardless of the boundary condition, for zero potential or if $q < \lambda w$ everywhere).
If $\tfrac{r}{w}$ is bounded and of bounded variation, the criterion simplifies to requiring $\tfrac{\abs{w'}}{w} \to \infty$ at singular boundary points.

This research is motivated by Slow Feature Analysis (SFA), a data processing algorithm that extracts the slowest uncorrelated signals from a high-dimensional input signal and has notable success in computer vision, computational neuroscience and blind source separation. From \cite{SprekelerZitoEtAl-2014} it is known that for an important class of scenarios (statistically independent input), an analytic formulation of SFA reduces to a Sturm-Liouville problem with zero potential and Neumann boundary conditions. So far, the mathematical SFA theory has only considered the regular case, except for a special case that is solved by Hermite Polynomials. This work generalizes SFA theory to the singular case, i.e.\ open-space scenarios.
\end{abstract}

\pagebreak

\section{Introduction} \label{sec:introduction}
Slow Feature Analysis (SFA) is an unsupervised data processing algorithm that has notable success in computer vision, computational neuroscience and blind source separation. It is based on the intuition that slowly varying signals are more likely to carry valuable information -- an intuition that can be understood to some extend by perceiving slow variation as an indicator for invariances. It is widely known that invariances are crucial for learning and recognition tasks.

The goal of this work is to extend the mathematical framework of SFA as generally as possible to what is called \textsl{the singular setting}. This setting permits the state space of SFA's input data being an unbounded domain or the probability distribution of the data vanishing at the boundary. Depending on the background of the reader, this setting may be better known as \textsl{open-space scenarios} but also covers various half-open spaces.
We begin with a summary of SFA and its mathematical framework, which we refer to as \textsl{SFA theory}\footnote{With ``SFA theory'' we refer to the mathematical framework for SFA, specifically the theory developed in \cite{doi:10.1162/089976603322297331, FranziusSprekelerEtAl-2007e, sprekeler2009slowness, doi:10.1162/NECO_a_00214, SprekelerZitoEtAl-2014}}. For convenience of the reader, a more detailed description of this matter is provided in \ref{sec:A_sfa}. Then we draw its connection with Sturm-Liouville theory and adopt the notion of a \textsl{singular setting} from that theory.
As a prerequisite for later proofs, we need to generalize monotonicity results from existing SFA theory. This is provided in Section\,\ref{sec:sfa_monotonicity} and readily stands on its own as one contribution of this work.
In Section\,\ref{sec:singular_sfa} we consider the aspects of SFA theory that require justification for the singular case. This justification is then mostly provided in Section\,\ref{sec:SFA_justification} with the exception of discreteness of the spectrum. For the complexity of the topic, discreteness of the spectrum is finally discussed in its own Section\,\ref{sec:discrete_spectrum}.

\subsection{Recapitulation of SFA} \label{sec:sfa}
Given an $n$-dimensional input signal $\vf{x}(t)$ in time, the objective of SFA is to extract a ``small'' number $m$ of signals that vary as slowly as possible over time. To avoid trivial constant solutions, the output is constrained to have unit variance and zero mean. Furthermore, the output signals must be pairwise uncorrelated to avoid redundant output components.
\cite{WiskottSejnowski-2002} describes this algorithm in detail. Originally, SFA performs feature extraction by a linear map, achieving non-linearity by a previous data expansion. In \cite{SprekelerWiskott-2008}, SFA is formulated for an infinite-dimensional, ``general'' function space $\mathcal{F}$ of sufficient regularity that replaces the two-step procedure for data extraction and enables analytical treatment. That forms the basis of a mathematical framework for SFA and gives rise to its extension \textsl{xSFA} in \cite{SprekelerZitoEtAl-2014} with an application to blind source separation. We sketch that mathematical foundation and key results from SFA theory. Assuming that $\vf{x}(t)$ is an ergodic process, SFA can be formulated in terms of the ensemble (i.e.\ the set of possible values of $\vf{x}$ and $\dot{\vf{x}}$) and its probability density $p_{\vf{x}, \dot{\vf{x}}}(\vf{x}, \dot{\vf{x}})$. If the input signal $\vf{x}(t)$ is composed of statistically independent sources $\vfc{s}{\alpha}$ for $\alpha \in \{1, \ldots, S \}$, one can break down the overall SFA problem into a separate one-dimensional SFA problem for each source. More precisely, there exists an injective transformation $\psi$ with ${\vf{x}(t) = \psi(\vf{s}(t))}$ (an embedding map). Given injectivity, $\psi$ is invertible on its image and we can likewise express $\vf{s}$ in terms of $\vf{x}$ as $\vf{s}(t) = \psi^{-1}(\vf{x}(t))$. Saying that ``$\vf{x}$ is composed of statistically independent sources'' then means nothing else than $\psi$ can be chosen such that the induced probability density of ${(\vf{s}, \dot{\vf{s}})}$ factorizes as ${p_{\vf{s}, \dot{\vf{s}}}(\vf{s}, \dot{\vf{s}}) = \prod_\alpha p_{s_\alpha, \dot{s}_\alpha}(s_\alpha, \dot{s}_\alpha)}$.

The terminology of ``sources'' originates from the SFA use case of blind source separation. A perhaps more intuitive view is to perceive the sources as coordinates of a manifold $\mathcal{V}$ embedded into the sensor space that hosts the input data. That view is taken in \cite{FranziusSprekelerEtAl-2007e} where the manifold $\mathcal{V}$ is the agent's state space or \textsl{configuration space}. This raises the question of how the notion of statistically independent sources transfers to coordinates of a manifold, so one can leverage the decomposition into separate one-dimensional SFA problems. For coordinates, ``statistically independent'' means that the position in one coordinate axis must not yield any information about the position in the others. This applies, for instance, to the agent's location in Cartesian coordinates in a rectangular room or $n$-box (assuming uniform probability distribution). As a counter example, consider the unit disc in Cartesian coordinates\footnote{One may consider polar coordinates statistically independent on the unit disc. However, the azimuth coordinate is of cyclic nature, so the corresponding source would lack continuity, which is a crucial assumption in SFA theory.}:
an $x$ value close to one would allow to conclude the $y$ value must be close to zero. Note that independence must hold for both the signal and its derivative, which is why in general for SFA it cannot be achieved by a coordinate transformation\footnote{Rather than an $n$-box, the state space may be a general manifold then, possibly involving non-zero curvature and lacking existence of a global parameterization. We investigate this setting in the next paper (in preparation) using Riemannian geometry and it turns out that the issue that statistical independence cannot generally be achieved by a coordinate transformation for SFA, is in fact a manifestation of the well known phenomenon that the curvature of a manifold cannot be transformed away by a coordinate change.} like described in \cite{685971}.

Regardless of whether one perceives ${\vf{s} = (\vfc{s}{1}, \ldots, \vfc{s}{S})^T}$ as sources or state space coordinates, it is straight forward via the chain rule to formulate SFA directly in terms of the ensemble ${(\vf{s}, \dot{\vf{s}})}$ and its probability density $p_{\vf{s}, \dot{\vf{s}}}(\vf{s}, \dot{\vf{s}})$. Statistical independence then allows to decompose the distribution and formulate SFA for each source separately in terms of the ensemble ${(s_\alpha, \dot{s}_\alpha)}$ and $p_{s_\alpha, \dot{s}_\alpha}(s_\alpha, \dot{s}_\alpha)$.

A central term in SFA theory is
\begin{equation} \label{K_def1D}
K_\alpha (s_\alpha) \quad \coloneq \quad \av{\vfc{\dot{s}}{\alpha}^2}_{\vfc{\dot{s}}{\alpha} | \vfc{s}{\alpha}}
\end{equation}
which represents the (first order) dynamics of the underlying agent or system producing the input data.
The notation in \eqref{K_def1D} refers to ${\av{f(s_\alpha, \dot{s}_\alpha)}_{\dot{s}_\alpha | s_\alpha} = \int p_{\dot{s}_\alpha | s_\alpha}(\dot{s}_\alpha|s_\alpha) f(s_\alpha, \dot{s}_\alpha) \; d \dot{s}_\alpha}$, based on the conditional probability density ${p_{\dot{s}_\alpha | s_\alpha}(\dot{s}_\alpha | s_\alpha) = \frac{p_{s_\alpha, \dot{s}_\alpha}(s_\alpha, \dot{s}_\alpha)}{p_{s_\alpha}(s_\alpha)}}$ with the marginal ${p_{s_\alpha}(s_\alpha) = \int p_{s_\alpha, \dot{s}_\alpha}(s_\alpha, \dot{s}_\alpha) \; d \dot{s}_\alpha}$.

\begin{definition}[1D SFA problem for a general (sufficiently regular) function space of a source $s_\alpha$] \label{def_1D_sfa}
The one-dimensional SFA problem associated with
a single statistically independent source $s_\alpha$, which takes on values in the interval $I_\alpha$, is
\begin{align}
\optmin{g_{\alpha i} \in \mathcal{F}}
&\int_{I_\alpha}  \;
g_{\alpha i}'^2(s_\alpha) \; K_\alpha(s_\alpha) && \!\!\!\!\!\!\!\!\!\!\!\!\!\!\!\!\!\!\!\!\!\!\!\!\!\!\!\!\!\!\!\!\!\!\!\!\!\!\!\!
\; p_\alpha(s_\alpha) \;
\; d s_\alpha \label{thm_1D_SFA_value_func} \\
\subjectto
&\int_{I_\alpha}  \;
g_{\alpha i}(s_\alpha) \; g_{\alpha j}(s_\alpha) && \!\!\!\!\!\!\!\!\!\!\!\!\!\!\!\!\!\!\!\!\!\!\!\!\!\!\!\!\!\!\!\!\!\!\!\!\!\!
p_\alpha(s_\alpha) \;
\; d s_\alpha \label{thm_1D_SFA_constr_decor}
\quad = \quad \delta_{ij} \qquad \forall j \leq i\\
&\int_{I_\alpha}  \;
g_{\alpha i}(s_\alpha) && \!\!\!\!\!\!\!\!\!\!\!\!\!\!\!\!\!\!\!\!\!\!\!\!\!\!\!\!\!\!\!\!\!\!\!\!\!\!
p_\alpha(s_\alpha) \;
\; d s_\alpha
\quad = \quad 0 \label{thm_1D_SFA_constr_zmean}
\end{align}
\end{definition}
The $i$th solution component $g_{\alpha i}$ for a source $s_\alpha$ is called \textsl{the $i$th harmonic of $s_\alpha$} in SFA theory, referring to harmonic oscillations which arise as solutions in the special case where $K_\alpha$ and $p_\alpha$ are constant. \textsl{Theorem 2} in \cite{SprekelerWiskott-2008}/\textsl{Theorem 1} in \cite{SprekelerZitoEtAl-2014} states that the solutions of the overall SFA problem then consist of all solutions of the individual sources and their products.

If an SFA problem can be decomposed into statistically independent sources $s_\alpha$, and if the interval of each source is finite and its probability density $p_\alpha(s_\alpha)$ is strictly positive, also at the boundary, then the SFA problem is relatively simple, as it can be solved as a set of regular (see below for a rigorous definition) one-dimensional SFA problems. SFA can be more complicated than that in at least two respects: (i) the interval of the sources may be infinite or $p$ vanishes somewhere leading to singular one-dimensional SFA problems or (ii) the problem cannot be decomposed into one-dimensional subproblems because the sources are not statistically independent of each other. (ii) leads to a multi-dimensional SFA problem that requires treatment on a manifold because the state space might not longer be an $n$-box. Case (i) is addressed in this paper; case (ii) will be addressed in a paper in preparation; (i) and (ii) combined is even more complicated and subject of future research.

\subsection{The Sturm-Liouville Problem}

On an open interval $(-\infty \leq a, b \leq \infty)$, let $L_{\text{loc}}(a, b)$ denote the space of locally Lebesgue-integrable functions and ${AC}_\text{loc}(a, b)$ the space of locally absolutely continuous functions, where \textsl{locally} means that the property must hold on every compact interval $J \subset (a, b)$.
These are standard terms from analysis and we skip the details of these definitions here. For the calculus in this paper it is enough to notice that absolutely continuous functions are also Lipschitz continuous, which implies that a function in ${AC}_\text{loc}(a, b)$ must be bounded\footnote{This follows because if a compact interval contains a pole of a function, there would not exist a suitable Lipschitz constant for the function on that interval.} on $(a, b)$ except possibly at the boundary points $a$ or $b$.

\begin{definition}[Sturm-Liouville problem] \label{SL-problem}
For a differential operator of the form\footnote{In most literature $r$ is denoted $p$. Here we avoid the use of $p$ because it denotes the probability density in SFA.}
\begin{equation} \label{SL-diff_op}
\mathcal{M} y(t) \quad = \quad \frac{1}{w(t)} \Big(-\big(r(t) y'(t)\big)' + q(t) y(t)\Big)
\end{equation}
on the domain $t \in (a, b)$ with $-\infty \leq a < b \leq \infty$ and $\tfrac{1}{r}, q, w \in L_{\text{loc}}(a, b)$, $w > 0$ almost everywhere, the problem of finding $y \in {AC}_\text{loc}(a, b)$ such that $r y' \in {AC}_\text{loc}(a, b)$ and
\begin{equation} \label{SL-problem_eq}
\mathcal{M} y(t) \quad = \quad \lambda y(t)
\end{equation}
is called a (possibly singular) Sturm-Liouville problem.

\eqref{SL-problem_eq} is called ``\textsl{regular}'' if $(a, b)$ is finite, $r$, $r'$, $w$, $q$ are continuous on $(a, b)$, $w, r > 0$ on $[a, b]$ and at every endpoint $c \in \{a, b\}$ it has boundary conditions of the form ${\alpha_{c,1} y(c) + \alpha_{c,2} r(c) y'(c) = 0}$ with $\alpha_{c,1}^2+\alpha_{c,2}^2 > 0$. Specifically, $w, r > 0$ must also hold at $a$ and $b$.
\eqref{SL-problem_eq} is called ``\textsl{singular}'' if it is not regular.
\end{definition}
\begin{rmk1}
The setup presented here is the standard setup from \cite{zettl2010sturm}.
\end{rmk1}

As noted above, $y, r y' \in {AC}_\text{loc}(a, b)$ implies that all solutions and their $r$-weighted derivatives are bounded on $(a, b)$ except possibly at the boundary points $a$ or $b$.
In this work, we only consider singular Sturm-Liouville problems with separated boundary conditions, i.e.\ the non-periodic case.

It is well known, e.g.\ \cite{courant1989methoden} Chapter\,VI, that the Sturm-Liouville problem in Definition\,\ref{SL-problem} arises via Euler-Lagrange equations from the variational optimization problem
\begin{align} \label{SL_opt_problem}
\optmin{y_i} & \int_a^b \; y_i'^2(t) r(t) \; + \; y_i^2(t) q(t) \; d t \\
\subjectto	& \int_a^b \; y_i(t) y_j(t) w(t) \; d t \quad = \quad \delta_{ij} \quad \forall \; j \leq i \label{SL_opt_problem_constraint}
\end{align}
and describes its critical points $y_i$ corresponding to the eigenvalues $\lambda_i$.
One easily observes the similarity with the SFA problem in a single source from Definition\,\ref{def_1D_sfa} and concludes that the case $q = 0$, $w = p_\alpha$, $r = K_\alpha p_\alpha$, $(a, b) = I_\alpha$, $\mathcal{F} = {AC}_\text{loc}(a, b)$ yields the SFA problem. Throughout this work, to indicate the special SFA setup versus a general Sturm-Liouville problem, we denote the unknown function as $g_\alpha(s_\alpha)$ for SFA and $y(t)$ for a general Sturm-Liouville problem. With this setup, in SFA theory $\mathcal{D}_\alpha$ denotes the differential operator $\mathcal{M}$ from \eqref{SL-diff_op}:
\begin{equation} \label{SFA_diff_op}
\mathcal{D}_\alpha \quad \coloneq \quad \partial_\alpha \big( p_\alpha(s_\alpha) K_\alpha(s_\alpha) \;
\partial_\alpha \vfc{g}{\alpha i}(s_\alpha) \big) \quad = \quad \big( p_\alpha(s_\alpha) K_\alpha(s_\alpha) \; \vfc{g'}{\alpha i}(s_\alpha) \big)'
\end{equation}
The $\partial_\alpha$ version is stated here to link with existing SFA theory literature, where this notation is used. For brevity, we use $'$ to denote derivation by $s_\alpha$ or $t$.
Throughout this work, for functions listed in Table\,\ref{table1} we frequently omit the independent variable.
\begin{table}[h!]
\begin{center}
\begin{tabular}{r|ccccc} 
\textbf{notation} & \textbf{weight} & \textbf{dynamics} & \textbf{potential} & \textbf{unknown} & \textbf{variable} \\
\hline
\textbf{SFA} & $p_\alpha(s_\alpha)$ & $K_\alpha(s_\alpha)$ & $0$ & $g_\alpha(s_\alpha)$ & $s_\alpha$ \\
\textbf{Sturm-Liouville} & $w(t)$ & $\tfrac{r(t)}{w(t)}$ & $q(t)$ & $y(t)$ & $t$
\end{tabular}
\caption{Notation overview of coefficient functions for SFA and Sturm Liouville problems. For functions in this table we often omit the independent variable to make equations more compact and readable.\label{table1}}
\end{center}
\end{table}

Problem \eqref{SL_opt_problem}/\eqref{SL_opt_problem_constraint} lacks the zero mean constraint \eqref{thm_1D_SFA_constr_zmean} of SFA. However, \eqref{thm_1D_SFA_constr_zmean} is automatically fulfilled by all solution components of \eqref{SL_opt_problem}/\eqref{SL_opt_problem_constraint} except the first, which is a constant with eigenvalue $\lambda = 0$. This is because subsequent components are mean free due to decorrelation with the constant. To turn a solution of problem \eqref{SL_opt_problem}/\eqref{SL_opt_problem_constraint} into an SFA solution one simply discards the first component.

Having no explicit boundary conditions, SFA is subject to the case called ``natural boundary conditions'' which leads to homogeneous Neumann boundary conditions, i.e.\
\begin{equation} \label{SL_Neumann_BC}
w(t) y_i'(t) \; \big|_{\{a, b\}} \quad = \quad 0
\end{equation}

The analogy with Sturm-Liouville problems suggests the distinction between regular and singular SFA problems:
we call an SFA problem \textsl{singular} if the corresponding Sturm-Liouville problem is singular.
Specifically, this corresponds to open-space scenarios in SFA, which are sometimes considered in existing SFA theory. One work dealing with an infinite state space is \cite{SprekelerWiskott-2011}, however in a non-standard SFA setup that permits eigenvalues in a continuous spectrum\footnote{The spectrum there consists of several continuous intervals. This is handled in equation (4.14) in that work, where the eigenvalue is expressed w.r.t.\ two parameters $q \in \mathbb{R}$ and $m \in \mathbb{N}^0$. $m$ selects an interval, and the ``frequency parameter'' $q$ parameterizes the selected interval. The setup is non-standard because the probability distribution is not required to be integrable.}.
In \cite{SprekelerWiskott-2008, SprekelerZitoEtAl-2014}, the special case of normally distributed $p_\alpha$ and constant $K_\alpha$ has been analyzed and the resulting SFA problem yields Hermite Polynomials as solutions. Despite being singular, it is known that this particular case behaves like the solution of a regular Sturm-Liouville problem in many ways, permitting major results from SFA theory to hold. 
In this work we investigate singular SFA for more general strictly positive $p_\alpha$ and $K_\alpha$, possibly zero at the boundary, permitting each source to be defined on an unbounded or half-bounded interval. This covers state spaces that are open spaces in some coordinate axes, half open in others and bounded in yet others.

To a significant extend, SFA theory depends on the rule that SFA delta values, i.e.\ the values of \eqref{thm_1D_SFA_value_func} at critical points, coincide with the eigenvalues. SFA delta values can be described as weighted Dirichlet energy\footnote{The Dirichlet energy of a function $f$ on a domain $\Omega$ is defined as $\tfrac{1}{2}\int_{\Omega} \norm{\nabla f(x)}^2 dx$. This equals \eqref{thm_1D_SFA_value_func} if the integral is weighted by $p_\alpha K_\alpha$ and the leading factor $\tfrac{1}{2}$ is dropped.} without leading factor $\tfrac{1}{2}$. In quantum mechanics this is also called the \textsl{kinetic energy} or -- since in SFA there is no potential -- the  \textsl{total energy}\footnote{The total energy consists of the kinetic energy and the potential energy. The latter is zero if the potential is zero, so for SFA total energy and kinetic energy are the same.}.
More generally, for regular Sturm-Liouville problems, it is known, again \cite{courant1989methoden} Chapter\,VI, that for each critical point, \eqref{SL_opt_problem} evaluates to the corresponding eigenvalue $\lambda_i$. One of the questions we investigate in this work is under what conditions this still holds for the singular case. We answer this question in Corollary\,\ref{cor_sfa_singular_justify}.

If the eigenvalues of a regular Sturm-Liouville problem form a discrete sequence, it is known that the eigenfunction belonging to the $i$th largest eigenvalue has exactly $i$ zeros. In the singular case this cannot be taken for granted and in fact it is not even guaranteed that the spectrum of eigenvalues is discrete at all. Both aspects are crucial assumptions in SFA theory and we identify conditions to justify them in Theorem\,\ref{thm_zeros_singular_SFA} and Theorem\,\ref{thm_SL_spectrum_discrete}.

\section{Monotonicity and Oscillation of SFA Solutions} \label{sec:sfa_monotonicity}
A crucial result from \cite{SprekelerWiskott-2008, SprekelerZitoEtAl-2014} states 
monotonicity of each first harmonic $\vfc{g}{\alpha 1}(\vfc{s}{\alpha})$ w.r.t.\ $\vfc{s}{\alpha}$. 
We generalize the statement to higher harmonics in the sense that increasing the index increases the number of stationary points exactly by one. Note that for the first harmonics, this proposition is slightly stronger than \textsl{strictly monotonic}\footnote{``strictly monotonic'' is defined as $x < y \Rightarrow g(x) < g(y)$. Note that this does not exclude saddle points.}, because it does not even permit saddle points\footnote{This is also apparent from the proof in \cite{SprekelerWiskott-2008, SprekelerZitoEtAl-2014}, but they do not advertise this feature explicitly. Their use case xSFA mainly requires the property of strict monotonicity; e.g.\ because it implies invertibility, so saddle points are not relevant there.}.
We further remark that it does not require $p_{\alpha}$ to be a probability density function but rather works for every strictly positive weighting function, i.e.\ it does not need to be normalized to one. Note that this is a proper generalization because for the SFA setting, $p_\alpha$ is assumed to be strictly positive\footnote{$p_\alpha$ vanishing somewhere in the inner of $I_\alpha$ would $g_{\alpha i}'$ permit to have poles which would invalidate the use of integration by parts or more generally of the Divergence Theorem on the integrand $g_{\alpha i}'^2 p_\alpha K_\alpha$. This would break various results from SFA theory. To fix this issue, one can exclude regions with vanishing $p_\alpha$ from the domain, resulting in additional boundary points and a possibly disconnected domain. If properly applied w.r.t.\ the redefined boundary, validity of the Divergence Theorem is restored this way. However, discussing the case of a disconnected domain is beyond the scope of this work. See \cite{everitt1986, doi:10.1137/0519078} for treatment of Sturm-Liouville theory on disconnected domains.} except possibly at the boundary (singular case). In standard form of a Sturm-Liouville problem, the SFA problem with its Neumann boundary condition is stated as
\begin{align}
\underbrace{\big(p_\alpha \; K_\alpha \; \vfc{g'}{\alpha i} \big)'}_{\qquad\quad\;\;\eq_\eqref{SFA_diff_op} \mathcal{D}_\alpha \vfc{g}{\alpha i}} \; + \;\; \lambda_{\alpha i} \; p_\alpha \; \vfc{g}{\alpha i} \quad &= \quad 0 \label{dgl-SL} \\
p_{\alpha} \; K_\alpha \; \vfc{g'}{\alpha i} \quad &= \quad 0 \quad\quad\quad \text{on the boundary} \label{neumann-s}
\end{align}

\begin{thm}[zeros and stationary points of SFA solution components interlace] \label{thm_sfa_interlace}
Let $g_{\alpha i}(s_\alpha)$ denote a solution of the differential equation \eqref{dgl-SL}. Assume $p_\alpha K_\alpha$ is bounded on $I_\alpha$, and $p_\alpha$ as well as $K_\alpha$ are each strictly positive except possibly at the boundary. Then, regardless of the boundary condition, one has $\lambda_{\alpha i} \geq 0$ and the following holds for all $g_{\alpha i}$ with $\lambda_{\alpha i} > 0$:
\begin{enumerate}
\item between every two zeros there is exactly one stationary point
\item between every two stationary points there is exactly one zero
\item $g_{\alpha i}$ and $p_\alpha K_\alpha g_{\alpha i}'$ cannot vanish at the same location on $I_\alpha$ and its boundary
\item there do not exist any saddle points
\end{enumerate}
This also holds if in \eqref{dgl-SL} $p_\alpha$ is not a probability density function but any strictly positive weighting function (of sufficient regularity, e.g.\ in $L_{\text{loc}}(I_\alpha)$), possibly zero at the boundary.
This theorem holds more generally for Sturm-Liouville problems with strictly positive $w$ and $r$ and for possibly non-zero $q$ if (but not only if) $q < \lambda w$ everywhere on the domain. Note that $\lambda$ may be negative for some solutions in case of non-zero $q$.
The domain $I_\alpha$ is not required to be bounded, i.e.\ the boundary may be at infinity.
\end{thm}
\begin{rmk}
Note that the first two items are not the same: the second one means that there are no multiple stationary points before the leftmost or after the rightmost zero.
This is especially relevant if a zero or a stationary point occurs at the boundary. Both cases are possible because we made no assumption on the boundary condition. Stationary points may occur at the boundary under the usual SFA boundary condition \eqref{neumann-s}. If instead homogeneous Dirichlet boundary conditions are employed (i.e.\ solutions must vanish at the boundary), $g_\alpha$ would have zeros at the boundary.
Since we assumed $p_\alpha K_\alpha$ bounded, Property\,3 implies that $g_\alpha$ cannot have a zero and a stationary point at the same location. This in turn implies a strong unique continuation\footnote{Strong unique continuation means that if a function vanishes of infinite order at a point, it must be zero everywhere on its domain. Given sufficient differentiability, a function vanishes of infinite order at a point if all of its derivatives are zero.
Strong unique continuation implies weak unique continuation, which means that if a function vanishes on an open set, then it vanishes everywhere on its domain. These are famous features of analytic functions.} property for SFA solution components. Further note that if $g_\alpha$ obeys the boundary condition \eqref{neumann-s}, it cannot vanish at the boundary, due to Property\,3.
\end{rmk}

\begin{proof}[Proof of Theorem\,\ref{thm_sfa_interlace}]
Let ${\xi_j \in I_\alpha}$ with $j \in \{1, \ldots, i\}$ denote the zeros of $g_{\alpha i}$. We conclude by \textsl{Rolle's theorem}\footnote{Rolle's theorem states that strictly between any two zeros of a continuously differentiable function $f$ there is at least one stationary point $x_0$, i.e.\ $f'(x_0) = 0$.} (special case of the \textsl{Mean Value Theorem}) that for $j \in \{1, \ldots, i-1\}$
\begin{equation} \label{SFA_monotonicity_proof_kappa}
\exists \; \kappa_{\alpha j} \in (\xi_{\alpha j}, \xi_{\alpha (j+1)}) \colon \qquad \vfc{g'}{\alpha i}(\kappa_{\alpha j}) \quad = \quad 0
\end{equation}
If there is a stationary point before the leftmost zero, let $\kappa_{\alpha 0}$ denote the leftmost stationary point, possibly at the boundary. If there is a stationary point after the rightmost zero, let $\kappa_{\alpha i}$ denote the rightmost stationary point, possibly at the boundary.
We apply the technique from the proof stated in \cite{SprekelerWiskott-2008, SprekelerZitoEtAl-2014} to each interval $(\xi_{\alpha j}, \kappa_{\alpha j})$ to prove that there is only one $\kappa_{\alpha j}$ that fulfills \eqref{SFA_monotonicity_proof_kappa} for a particular $j \in \{1, \ldots, i \}$. The proof applies analogously to the intervals $(\kappa_{\alpha (j-1)}, \xi_{\alpha j})$.
Without loss of generality we perform the proof for a $j$ such that ${\vfc{g}{\alpha i} > 0}$ for $\vfc{s}{\alpha} \in (\xi_{\alpha j}, \kappa_{\alpha j}]$. Since we cannot assume a-priori that only a single $\kappa_{\alpha j}$ fulfills \eqref{SFA_monotonicity_proof_kappa}, we assume to have picked the rightmost, i.e.\ largest, $\kappa_{\alpha j}$ in $(\xi_{\alpha j}, \xi_{\alpha (j+1)})$. We show that $\vfc{g'}{\alpha i} > 0$ on $[\xi_{\alpha j}, \kappa_{\alpha j})$, which rules out saddle points on $[\xi_{\alpha j}, \kappa_{\alpha j})$ and implies that
$\vfc{g}{\alpha i}$ is strictly monotonic on $[\xi_{\alpha j}, \kappa_{\alpha j})$, ruling out the possibility of further $\kappa_{\alpha j}$ in terms of \eqref{SFA_monotonicity_proof_kappa}. The left boundary point $\xi_{\alpha j}$ is included to prove Property\,3 (c.f.\ \eqref{lem_SFA_mono_gen_eq2}). The fact that saddle points are ruled out, proves\ Property\,4.
For the moment, we assume $\lambda_{\alpha i} > 0$ and postpone the proof of this statement.
\begin{align}
\eqref{dgl-SL} \;\; \ra_{\hphantom{\eqref{SFA_monotonicity_proof_kappa}}}&\quad
\big(p_{\alpha} K_\alpha \vfc{g'}{\alpha i} \big)' \quad = \quad  -\underbrace{\lambda_{\alpha i} p_{\alpha}}_{> 0} \vfc{g}{\alpha i} \quad < \quad 0 \quad \forall \; \vfc{s}{\alpha} \in (\xi_{\alpha j}, \kappa_{\alpha j}) \label{sfa_mon_ansatz} \\
\ra_{\hphantom{\eqref{SFA_monotonicity_proof_kappa}}}&\quad
p_{\alpha} K_\alpha \vfc{g'}{\alpha i} \quad \text{strictly monotonically decreasing on } [\xi_{\alpha j}, \kappa_{\alpha j}] \label{sfa_mon_decr_bounded}\\
\ra_{\eqref{sfa_mon_ansatz}}&\quad
\! \underbrace{p_{\alpha} K_\alpha}_{\geq 0} \vfc{g'}{\alpha i} \; > \; 0 \; \text{on } [\xi_{\alpha j}, \kappa_{\alpha j}) \; \text{as \eqref{sfa_mon_decr_bounded} is bounded below by} \;0
\eq_{\eqref{SFA_monotonicity_proof_kappa}}
\vfc{g'}{\alpha i}(\kappa_{\alpha j}) \label{sfa_mon_decr_bounded_cons} \\
\ra_{\hphantom{\eqref{SFA_monotonicity_proof_kappa}}}&\quad
\vfc{g'}{\alpha i} \quad > \quad 0 \quad \text{on} \quad [\xi_{\alpha j}, \kappa_{\alpha j}) \label{sfa_mon_decr_bounded_cons2} \\
\ra_{\hphantom{\eqref{SFA_monotonicity_proof_kappa}}}&\quad
\vfc{g}{\alpha i} \quad \text{strictly monotonically increasing on } [\xi_{\alpha j}, \kappa_{\alpha j}] \label{sfa_mon_decr_final}
\end{align}
The case $p_{\alpha} K_\alpha = 0$ in \eqref{sfa_mon_decr_bounded_cons} can only occur if $\xi_{\alpha j}$ is a boundary point. The conclusion to \eqref{sfa_mon_decr_bounded_cons2} is still valid then because it would rather cause $\vfc{g}{\alpha i}$ to be positive infinity rather than zero. To invalidate the conclusion that $\vfc{g}{\alpha i}$ is strictly positive at $\xi_{\alpha j}$, $p_{\alpha} K_\alpha$ would have to become infinite there which is prohibited by assumption.
An analogous calculation\footnote{``bounded below'' may need to be replaced by ``bounded above'' and ``decreasing'' by ``increasing'' or vice versa} holds if we assume
${\vfc{g}{\alpha i} < 0}$ for $\vfc{s}{\alpha} \in (\xi_{\alpha j}, \kappa_{\alpha j}]$ or if we consider the intervals $(\kappa_{\alpha (j-1)}, \xi_{\alpha j}]$ rather than $[\xi_{\alpha j}, \kappa_{\alpha j})$.
These various cases prove Properties~1 and 2.
Property\,3 follows because by \eqref{sfa_mon_decr_bounded_cons} $p_{\alpha} K_\alpha \vfc{g'}{\alpha i}$ cannot vanish at $\xi_{\alpha j}$, not even if $\xi_{\alpha j}$ is at the boundary. Since $\xi_{\alpha j} \neq \kappa_{\alpha j}$ and $\vfc{g}{\alpha i}(\xi_{\alpha j}) = 0$, we conclude by \eqref{sfa_mon_decr_final} that $\vfc{g}{\alpha i}$ cannot vanish at $\kappa_{\alpha j}$.

To conclude $\lambda_{\alpha i} > 0$, consider \eqref{sfa_mon_ansatz}-\eqref{sfa_mon_decr_final} with negative $\lambda_{\alpha i}$. This would flip all inequalities and the outcome in \eqref{sfa_mon_decr_final} would be that $\vfc{g}{\alpha i}$ is strictly monotonically \textsl{decreasing} on $[\xi_{\alpha j}, \kappa_{\alpha j}]$. This would be a contradiction to the assumption that $\vfc{g}{\alpha i} > 0$ on $(\xi_{\alpha j}, \kappa_{\alpha j}]$ since $\vfc{g}{\alpha i}(\xi_{\alpha j}) = 0$. Analogous contradictions occur if $\vfc{g}{\alpha i} < 0$ on $(\xi_{\alpha j}, \kappa_{\alpha j}]$ or if one considers the intervals $(\kappa_{\alpha (j-1)}, \xi_{\alpha j}]$.

The whole calculation does not require $p_{\alpha}$ to be a probability density function, but only to be a strictly positive weighting function. To consider a generic Sturm-Liouville problem, possibly with non-zero $q$, one replaces $\lambda_{\alpha i} p_{\alpha}$ by $(\lambda w - q)$ in \eqref{sfa_mon_ansatz} and uses that $(\lambda w - q)$ is required to be strictly positive for this case.
\end{proof}

In the following corollary of Theorem\,\ref{thm_sfa_interlace} we formulate the matter as a proper generalization of a result from existing SFA theory.

\begin{corr}[monotonicity of first harmonics for statistically independent input] \label{cor_sfa_monotonic}
(generalization of result from \cite{SprekelerWiskott-2008, SprekelerZitoEtAl-2014})

If $\vf{s}$ consists of statistically independent components $\vfc{s}{\alpha}$ like introduced above, let $\vfc{s}{\alpha}$ take on values in ${I_{\alpha} = (a_{\alpha}, b_{\alpha})}$. By Sturm-Liouville Theory, the $i$-th SFA harmonic $\vfc{g}{\alpha i}(\vfc{s}{\alpha})$
of each source $\vfc{s}{\alpha}$ has exactly $i$ zeros on $I_{\alpha}$. Let $\xi_{\alpha j} \in I_{\alpha}$ with $j \in \{1, \ldots, i\}$ denote these zeros ordered from left to right, i.e.\ by quantity in $I_{\alpha}$. Then, assuming that $p_\alpha$ and $K_\alpha$ obey the assumptions of Theorem\,\ref{thm_sfa_interlace},
\begin{align}
&\vfc{g}{\alpha i} \quad \text{has exactly one stationary point on} \quad (\xi_{\alpha j}, \xi_{\alpha (j+1)})
\quad \forall \; j \in \{1, \ldots, i-1\} \label{lem_SFA_mono_gen_eq} \\
&\vfc{g}{\alpha i} \quad \text{has no stationary point at} \quad \xi_{\alpha j}
\quad \forall \; j \in \{1, \ldots, i\}. \label{lem_SFA_mono_gen_eq2} \\
&\vfc{g}{\alpha 1} \quad  \text{is a strictly monotonic signal of the source} \;\; \vfc{s}{\alpha} \; \text{without saddlepoints}. \label{cor_sfa_monotonic_eq}
\end{align}
\end{corr}
\begin{rmk1}
Corollary\,\ref{cor_sfa_monotonic} still holds under the weakened assumptions of Theorem\,\ref{thm_sfa_interlace} as long as one sticks to the Neumann boundary condition. I.e.\ $p_{\alpha}$ can be any bounded, strictly positive weighting function (of sufficient regularity, e.g.\ in $L^1$) and a non-zero potential $q$ is permitted if $q < \lambda w$ holds everywhere on $I_{\alpha}$.
\end{rmk1}

\begin{proof}[Proof of Corollary\,\ref{cor_sfa_monotonic}]
With Sturm-Liouville theory stating that $\vfc{g}{\alpha i}$ has exactly $i$ zeros
we conclude by Theorem\,\ref{thm_sfa_interlace} that it has $i-1$ stationary points and additionally those at the boundary, possibly induced by \eqref{neumann-s}. Property\,3 from Theorem\,\ref{thm_sfa_interlace} states \eqref{lem_SFA_mono_gen_eq2}.
Considering the Neumann boundary condition \eqref{neumann-s} of the SFA setting and since $\vfc{g}{\alpha 1}$ has only a single zero, for $i = 1$ no intervals $(\xi_{\alpha j}, \xi_{\alpha (j+1)})$ exist where \eqref{lem_SFA_mono_gen_eq} would apply, so $\vfc{g}{\alpha 1}$ is at most stationary at the boundary $\{a_{\alpha}, b_{\alpha}\}$.
\end{proof}

\section{SFA on Unbounded Domains or in the Singular Case} \label{sec:singular_sfa}
\cite{SprekelerWiskott-2008, SprekelerZitoEtAl-2014} apply their theory to normally distributed sources on an infinite interval, yielding Hermite polynomials as SFA solutions. This requires some justification as an infinite interval, or similarly, a probability density function that vanishes at the boundary, is considered a \textsl{singular case} in Sturm-Liouville theory. Classical results from Sturm-Liouville theory are not guaranteed to hold in the singular case, such as the statement that the $n$th eigenfunction has exactly $n$ zeros (Corollary\,\ref{cor_sfa_monotonic} depends on this property but Theorem\,\ref{thm_sfa_interlace} does not). Hermite polynomials are a particular example where a singular Sturm-Liouville problem still possesses various of the properties we use to know from the regular case. But what if $p$ is not normally distributed but some other probability density function over the reals? It is known that the spectrum of eigenvalues $\lambda_i$ is not necessarily discrete in the singular case and may contain one or several, possibly an infinite amount of, continuous intervals.

Another issue is in the proof that SFA eigenvalues coincide with SFA delta values. One computes SFA delta values via integration by parts:
\begin{equation} \label{S16_A27}
\int_a^b \; g_{\alpha i}'^2 \; K_\alpha \; p_\alpha \; d s_\alpha \quad = \quad
\underbrace{\Big[ \underbrace{g_{\alpha i}' \; K_\alpha \; p_\alpha}_{\Lra_{\eqref{SL_Neumann_BC}} 0} \; \underbrace{g_{\alpha i}}_{\Lra \infty?} \; \Big]_a^b}_{\Lra \; ?}
\quad - \quad \int_a^b g_{\alpha i} \; \underbrace{\big(g_{\alpha i}' \; K_\alpha \; p_\alpha \big)'}_{\eq_\eqref{dgl-SL} \lambda_i p_\alpha g_{\alpha i}} \; d s_\alpha \quad \eq^? \quad \lambda_i
\end{equation}
In the regular case, this equation is determinate because $g_{\alpha i}$ is finite at $a$ and $b$. Occurrences of this equation in (S16), supplementary material of \cite{FranziusSprekelerEtAl-2007e}, proofs of Theorems~2 and 4 or similarly in (A27) \cite{sprekeler2009slowness}, proofs of Lemmas~2 and 3 in A.1, do not cover the singular case, i.e.\ where the SFA solution component, denoted $f$ in (S16) and (A27), is possibly unbounded.
\eqref{S16_A27} is used in SFA theory to show that SFA eigenvalues are positive and SFA delta values are finite.
As a consequence, these intuitive properties cannot be taken for granted in the singular case.
Also the self-adjointness of the differential operator \eqref{SFA_diff_op} is based on equation (S16) or (A27) respectively. \eqref{SFA_diff_op} being self-adjoint is the basis for the proof of \eqref{dgl-SL} and \eqref{A_dgl-x}
by means of the Spectral Theorem. \textsl{Theorem 2} in \cite{SprekelerWiskott-2008}/\textsl{Theorem 1} in \cite{SprekelerZitoEtAl-2014} is also based on these properties, c.f.\ \eqref{A_dgl-x-decomp} which is based on \eqref{A_dgl-x}. That means, more or less the whole SFA theory is doubtful for the singular case.

In the following sections we are able to fully justify results from known SFA theory for the singular case. This includes the cases where the domain $I_\alpha$ of a source $s_\alpha$ is unbounded on one or both sides (the \textsl{half line} and the \textsl{full line}). In addition, we can provide a simple criterion that guarantees a discrete spectrum of eigenvalues. Even if the spectrum is not fully discrete, we can justify known SFA theory for the discrete part of the spectrum.

It is remarkable that a rather simple criterion for discreteness of the SFA spectrum can be given, because as of today, this is an unsolved question for the general Sturm-Liouville problem, see \cite{zettl2010sturm} Chapter\,10.13. Based on recent work of \cite{ROMANOV2020108318}, we can deduce that for an SFA problem, the spectrum is discrete if and only if $\tfrac{p'\!\!}{p}$ is infinite at every singular boundary point, i.e.\ at $\infty$ or $-\infty$ in case of the half line or at both of them in case of the full line. This criterion is easy to check for a given probability density function, especially since $\tfrac{p'\!\!}{p} = (\log{p})'$. E.g.\ it immediately confirms discreteness of the spectrum for the case of Hermite polynomials.

Before we continue, we illustrate that singularness of an SFA problem cannot be transformed away by changing the independent variable.
One may consider the length functional
\begin{equation} \label{length_F_p}
L_{t_0, t}(\gamma) \quad \coloneq \quad \int_{t_0}^t
\;\; \gamma(\theta) \;
\; d \theta
\end{equation}
to transform the unbounded or half-bounded interval into a bounded interval by
parameterizing by $p_\alpha$-weighted arc length.
Based on \eqref{length_F_p} one defines the diffeomorphism $\phi$, setting $\tilde{s}_\alpha \coloneq \phi(s_\alpha)$:
\begin{align}
\phi(s_\alpha) \!\quad \coloneq& \quad L_{a_\alpha, s_\alpha}(p_\alpha) \label{phi1_p} \\
\ra_\eqref{length_F_p} \qquad \phi'(s_\alpha) \quad =& \quad p_{\alpha}(s_\alpha) \label{phi2_p} \\
\ra \;\;\;\; {\phi^{-1}}'(\tilde{s}_\alpha) \quad =& \quad \frac{1}{\phi'(\phi^{-1}(\tilde{s}_\alpha))} \quad \eq_\eqref{phi2_p} \quad \frac{1}{\; p_{\alpha}(\phi^{-1}(\tilde{s}_\alpha))} \label{phi3_p}
\end{align}
Observe that $\phi(a_\alpha) = 0$ and $\phi(b_\alpha) = 1$, i.e.\ $\phi$ parameterizes the unit interval over $I_{\alpha}$.
It is straight forward to parametrize the SFA value functional \eqref{thm_1D_SFA_value_func} by $\phi$:
\begingroup
\allowdisplaybreaks
\begin{align}
\int_{a_\alpha}^{b_\alpha} {g_\alpha'}^2
\; K_\alpha \; p_\alpha
\;\; d s_\alpha \quad
\eq_{\eqref{phi1_p}}&\quad
\int_{\phi(a_\alpha)}^{\phi(b_\alpha)} \;\, (g_\alpha' \circ \phi^{-1})^2 \;
\; K_\alpha \circ \phi^{-1} \; \; \underbrace{p_\alpha \circ \phi^{-1} \;
\;\; {\phi^{-1}}'}_{\hphantom{1} \; = \, 1}
\;\; d \tilde{s}_\alpha
\\
\eq_{\substack{\eqref{phi2_p} \\ \eqref{phi3_p}}}& \quad
\int_0^1 \;\, ({g_\alpha'} \circ \phi^{-1})^2 \; ({\phi^{-1}}')^2
\; \frac{K_\alpha \circ \phi^{-1}}{({\phi^{-1}}')^2}
\;\; d \tilde{s}_\alpha
\\
\eq_{\eqref{phi3_p}}& \quad
\int_0^1 \;\, (\underbrace{g_\alpha \circ \phi^{-1}}_{\quad\;\eqcolon \tilde{g}_\alpha})'^2 \;
\; \underbrace{K_\alpha \circ \phi^{-1} \; \; p^2_\alpha \circ \phi^{-1}}_{\quad\;\eqcolon \tilde{K}_\alpha}
\;\; d \tilde{s}_\alpha
\\
\eq_{\hphantom{\eqref{phi3_p}}}& \quad
\int_0^1 \;\, \tilde{g}_\alpha'^2 \;\; \tilde{K}_\alpha
\;\; d \tilde{s}_\alpha
\end{align}
\endgroup
Since $p_\alpha$ is a probability density function, it must have a finite integral over $I_\alpha$. If $I_\alpha$ is infinite and the limit of $p_\alpha$ exists at a particular infinite endpoint, $p_\alpha$ must become zero at that endpoint (c.f.\ \textsl{Barb\u{a}lat's Lemma}). So ${\tilde{K}_\alpha = (K_\alpha p^2_\alpha) \circ \phi^{-1}}$ must do as well, assuming $K_\alpha$ is bounded. This coefficient function becoming zero at the boundary falls still into the singular case.
That means, the singular case cannot be dealt with merely by transforming the problem, at least not in this manner. We leverage a number of results from singular Sturm-Liouville theory to handle this case.

\section{Justification of SFA Theory for the Singular Case} \label{sec:SFA_justification}
It is known that singular endpoints of an interval can be of several types that allow further conclusion about the extend to which solutions possess regular properties. The main distinction is between \textsl{limit~point}~(LP) and \textsl{limit~circle}~(LC) endpoints. An endpoint is LC if the solutions are in $L^2$, i.e.\ Lebesgue integrable in squared value, for every neighborhood of the singular endpoint in question, and LP otherwise.
A further distinction is between oscillatory (O) and non-oscillatory (NO) endpoints, where oscillatory means that the zeros of a solution accumulate at the endpoint in question. See \cite{zettl2010sturm}, Definition\,7.3.1 for a reference of these cases and Chapters 6 and 7 for a deeper discussion of the properties associated with each case and criteria to detect the various cases.
By \cite{niessen1992singular} Theorem\,4.1, it is known that in the LC case, the question of O versus NO is independent of $\lambda$.
In that work it is further shown that LCNO endpoints yield solutions that possess the usual properties we know from the regular case. I.e.\ if both endpoints of the domain are regular or LCNO, the problem essentially behaves like a regular Sturm-Liouville problem. For a broader overview of (singular) Sturm-Liouville theory and of the LP, LC, O, NO properties we refer the reader to Anton Zettl's book \cite{zettl2010sturm}. Of particular interest from that book for us is the summarizing Theorem\,10.12.1. Section\,(8) of that theorem deals with the LP case and asserts regular properties of the solution if the spectrum of eigenvalues is discrete and bounded below (BD\footnote{Referring to a spectrum, \textsl{BD} is the established acronym for ``discrete and bounded below'' in common literature.}). If the spectrum is bounded below but contains a continuous part, regular properties still apply to solutions with eigenvalues from the discrete part of the spectrum $\sigma_{\text{d}}$. The continuous part is denoted the \textsl{essential spectrum} $\sigma_{\text{ess}}$ and eigenfunctions belonging to eigenvalues from $\sigma_{\text{ess}}$ have an infinite number of zeros. If the spectrum is not bounded below, every eigenfunction has an infinite number of zeros. Theorem\,10.12.1 lists results from various sources. The proof for the LCNO case can be found in \cite{niessen1992singular} and the results regarding the LP case refer to \cite{weidmann1987spectral}, where they are stated in Theorems 14.9 and 14.10 with proofs.
These properties allow us to conclude the following theorem for SFA problems:

\begin{thm}[zeros of singular SFA problems] \label{thm_zeros_singular_SFA}
Set $\sigma_0 \coloneq \inf{\sigma_{\text{ess}}}$ with $\sigma_0 = \infty$ if no essential spectrum exists, e.g.\ for regular SFA problems.
The SFA solution component belonging to the $i$th discrete eigenvalue $\lambda_i < \sigma_0$ of a regular or singular SFA problem has exactly $i$ zeros on its domain.
\end{thm}

\begin{corr}[singular case of Corollary\,\ref{cor_sfa_monotonic}] \label{singular_case_lem1}
Assuming that $p_\alpha$ and $K_\alpha$ obey the assumptions of Theorem\,\ref{thm_sfa_interlace},
Corollary\,\ref{cor_sfa_monotonic} also holds for the singular case.
\end{corr}

\begin{proof}[Proof of Corollary\,\ref{singular_case_lem1}]
By Theorem\,\ref{thm_zeros_singular_SFA}, the $i$th solution component of an SFA problem has $i$ zeros also in the singular case. This is the only property required in the proof of Corollary\,\ref{cor_sfa_monotonic} that needed justification for the singular case.
\end{proof}

The theory mentioned above only applies to self-adjoint problems. So, before we can prove Theorem\,\ref{thm_zeros_singular_SFA}, we must justify self-adjointness of the SFA problem for the singular case.
\begin{lem}[SFA problems are self-adjoint, also in the singular case] \label{lem_singular_SFA_self_adjoint}
(generalization of \cite{FranziusSprekelerEtAl-2007e}/suppl./Theorem 2 or \cite{sprekeler2009slowness}/A.1/Lemma\,2)

The differential operator $\mathcal{D}_{\alpha}$ from \eqref{SFA_diff_op} under the boundary condition \eqref{neumann-s} is self-adjoint, even if one or both endpoints of the domain $I_\alpha$ are singular.
\end{lem}
The regular case has been proven in supplementary material of \cite{FranziusSprekelerEtAl-2007e}, Theorem\,2 or similarly in \cite{sprekeler2009slowness}, Lemma\,2 in A.1. To prove the singular case we leverage results from singular Sturm-Liouville theory that provide conditions for self-adjointness.
The following self-adjointness criterion for regular and singular Sturm-Liouville problems is known:

\begin{thm}[criterion for a (possibly singular) Sturm-Liouville problem to be self-adjoint] \label{SL_self-adjoint}
(from \cite{krall1988, zettl2010sturm})

For $r$ from \eqref{SL-diff_op} and differentiable functions $f$, $g$, the \textsl{Wronskian} is given as
\begin{equation} \label{wronskian}
W(f, g) \quad = \quad f r g' - g r f'
\end{equation}
Let $u, v$ denote solutions of \eqref{SL-problem_eq} for $\lambda = 0$ with $W(u, v) = 1$.
For a regular or LC endpoint $c \in \{a, b\}$ set
\begin{equation} \label{Y_cases}
Y(t) \quad = \quad
\begin{cases}
c \; \text{regular} \colon & \quad \big(y(t), \; r(t) y'(t) \big)^T \\
c \; \text{LC} \colon & \quad \big(W(y, v)(t), \;W(y, u)(t) \big)^T
\end{cases}
\end{equation}
Consider the boundary conditions regarding matrices $\vf{A}, \vf{B} \in \mathbb{R}^{2 \times 2}$:
\begin{equation} \label{singular_bc}
\vf{A} Y(a) \; + \; \vf{B} Y(b) \quad = \quad 0
\end{equation}
where $Y(a) = \lim\limits_{x \searrow a} Y(x)$ and $Y(b) = \lim\limits_{x \nearrow b} Y(x)$.
At an LP endpoint set $\vf{A} = \vf{0}$ or $\vf{B} = \vf{0}$, respectively.

Then problem \eqref{SL-problem_eq} is self-adjoint if and only if
\begin{equation} \label{self-adjoint_det}
\det \vf{A} \quad = \quad \det{\vf{B}}
\end{equation}
wherein $\vf{A}$, $\vf{B}$ must be non-zero at a single non-LP endpoint or in case of two non-LP endpoints the $2 \times 4$ block matrix $(\vf{A}, \vf{B})$ must have rank $2$.
\end{thm}
\begin{rmk}
In this setting, the Wronskian equals the \textsl{Lagrange bracket} or \textsl{Lagrange sesquilinear form}, denoted $[f, g]$. In \cite{zettl2010sturm}, this notation is used instead, see equation 10.2.3. The theorem also exists for ${A, B \in \mathbb{C}^{2 \times 2}}$, but then $A$, $B$ must satisfy an additional condition to \eqref{self-adjoint_det}, given in \cite{krall1988}.
\end{rmk}
Theorem\,\ref{SL_self-adjoint} is a combined formulation of Theorems 1-4 from \cite{krall1988}, or similarly of Theorems 10.4.2, 10.4.4, 10.4.5 and Proposition 10.4.2 from \cite{zettl2010sturm}.
We refer the reader to these sources for proofs and background information.
Note how popular special cases are contained in Theorem\,\ref{SL_self-adjoint}:
\begin{itemize}
\item separated boundary conditions: $\det \vf{A} = \det \vf{B} = 0$
\item periodic boundary conditions: $\vf{A} = -\vf{B} = \vf{I}$
\item antiperiodic boundary conditions: $\vf{A} = \vf{B} = \vf{I}$
\end{itemize}
That means, problems with at least one LP endpoint can never have coupled boundary conditions.

\begin{proof}[Proof of Lemma\,\ref{lem_singular_SFA_self_adjoint}]
We apply Theorem\,\ref{SL_self-adjoint} to the SFA problem by setting $q = 0$, $w = p_\alpha$, $r = K_\alpha p_\alpha$, $(a, b) = I_\alpha$. The solutions with $\lambda = 0$ and some $x_0 \in I_\alpha$
\begin{equation} \label{lambda0_solutions}
u(t) \; \equiv \; 1, \quad v(t) \; = \; \int_{x_0}^t \; \frac{1}{K_\alpha(s_\alpha) p_\alpha(s_\alpha)} \; d s_\alpha
\end{equation}
then yield
\begin{equation}
Y \quad = \quad
\begin{pmatrix}
K_\alpha p_\alpha g_\alpha' \\
g_\alpha - v K_\alpha p_\alpha g_\alpha'
\end{pmatrix}
\end{equation}
at an LC endpoint according to \eqref{Y_cases} with $y=g_\alpha$.
For regular or LC endpoints we choose
\begin{equation}
\vf{A} \; = \;
\begin{pmatrix}
1 & 0 \\
0 & 0
\end{pmatrix}
, \quad
\vf{B} \; = \;
\begin{pmatrix}
0 & 0 \\
1 & 0
\end{pmatrix}
\end{equation}
Then \eqref{singular_bc} results in the usual separated SFA boundary condition \eqref{neumann-s} for all endpoints that require boundary conditions at all, i.e.\ non-LP. Consequently, according to Theorem\,\ref{SL_self-adjoint}, the original SFA formulation yields self-adjoint problems also in the singular case.
\end{proof}
\begin{rmk1}
The solutions $u$, $v$ in \eqref{lambda0_solutions} work for every Sturm-Liouville problem with $q = 0$ and are a typical choice for this case. See e.g.\ Examples I$\,$-V in \cite{krall1988}.
\end{rmk1}

\begin{proof}[Proof of Theorem\,\ref{thm_zeros_singular_SFA}]
For regular SFA problems, the statement follows from classical Sturm-Liouville theory.
In the singular case, Lemma\,\ref{lem_singular_SFA_self_adjoint} asserts that the differential operator $\mathcal{D}_{\alpha}$ from \eqref{SFA_diff_op} is still self-adjoint, so \cite{zettl2010sturm} Theorem\,10.12.1 is applicable.
We note that in principle the SFA problem permits the constant solution for $\lambda = 0$. This solution is artificially prohibited by the zero mean constraint, but is suitable to show that SFA problems are never oscillatory at a singular endpoint: for the constant solution, no zeros exist at all, so they cannot accumulate at an endpoint, not even at a singular one. (The constant solution then is a \textsl{principal solution} in terms of \cite{niessen1992singular}, Definition\,2.1 and the subsequent theory applies.) Since in the LC case this property is independent of $\lambda$ \cite{niessen1992singular} Theorem\,4.1, SFA problems are never LCO. They can, however, be LP and the Hermite polynomials are an example that is well known to be LP at both endpoints. If an SFA problem is not LP, it must be LCNO and behave like a regular Sturm-Liouville problem since \cite{niessen1992singular} Theorems 5.2 and 5.3 apply. These prove the claim for the LC case.
For the LP case we consider \cite{zettl2010sturm} Theorem\,10.12.1, (8) (see \cite{weidmann1987spectral} for detailed analysis and proofs of this matter).
It distinguishes the cases (i) $\sigma_0 = -\infty$, (ii) $\sigma_0 = \infty$, and (iii) $-\infty < \sigma_0 < \infty$. About case (i) it states that every eigenfunction has an infinite number of zeros. Since we know that for an SFA problem there always exists a solution with no zeros at all, namely the constant solution, an SFA problem can never be in case (i). In case (ii), the spectrum is discrete and a further distinction is made whether it is bounded below. If it is unbounded below, every eigenfunction has an infinite number of zeros, so we can exclude this case by the same argument as for case (i). If the spectrum is bounded below, every eigenvalue is simple (for SFA it is well known that eigenvalues might not be simple in case of periodic boundary conditions; above we pointed out that in the LC case no coupled boundary conditions, e.g.\ periodic or antiperiodic, can apply) and the eigenfunction belonging to the $i$th eigenvalue has exactly $i$ zeros, so the statement is proven.
In case (iii) there may be a finite or an infinite number of discrete eigenvalues $\lambda_i < \sigma_0$ or none at all. For these eigenvalues it is stated that the eigenfunction belonging to the $i$th eigenvalue has exactly $i$ zeros, so the statement is proven. All possible cases have been considered.
\end{proof}

Before we move on, we need to justify existence of the limits of an SFA solution component and its derivative at singular endpoints. Here, \textsl{existence} does not refer to finiteness; it merely excludes divergence in terms of bounded or unbounded infinite oscillation. In this work we generally refer to \textsl{existing limits} by means of existing in ${\bar{\mathbb{R}} = \mathbb{R} \cup \{-\infty, \infty\}}$. Since in literature, \textsl{existing} often refers to \textsl{existing in $\mathbb{R}$} we remind the reader of this choice by stating ``existing but possibly infinite.''
With the following lemma we give sufficient conditions for this existence of some relevant limits solely in terms of the coefficients $p_\alpha$ and $K_\alpha$:
\begin{lem}[existence of limits of SFA solutions] \label{lem_limits_SFA}
Assuming that $p_\alpha$ and $K_\alpha$ obey the assumptions of Theorem\,\ref{thm_sfa_interlace},
let $g_\alpha(s_\alpha)$ denote a solution component of a possibly singular SFA problem for a discrete eigenvalue $0 < \lambda < \sigma_0$ regarding the source $s_\alpha$ on the domain $I_\alpha$. Let $a$ denote a possibly singular endpoint of $I_\alpha$.
\begin{enumerate}
\item The limit of $g_\alpha$ exists at $a$ and is non-zero, possibly infinite though.
\item If $K_\alpha$ is non-zero and $\tfrac{(p_\alpha K_\alpha)'}{p_\alpha K_\alpha}$ infinite at $a$ (e.g.\ because $\tfrac{p_\alpha'}{p_\alpha}$ is infinite at $a$) and the limit of $g_\alpha'$ exists, then the limit of $\tfrac{g_\alpha'}{g_\alpha}$ exists at $a$, possibly infinite though.
\item If the limit $\lim\limits_{s_\alpha \to a} p_\alpha K_\alpha$ does \textsl{not} exist or is non-zero, then the limit of $g_\alpha'$ exists at $a$ and is zero. Consequently, also the limit of $\tfrac{g_\alpha'}{g_\alpha}$ exists and is zero.
\end{enumerate}
\end{lem}
\begin{rmk}
Note that in Statement\,3 the case of $\lim\limits_{s_\alpha \to a} p_\alpha K_\alpha$ non-zero existing is hardly feasible for SFA: since $p_\alpha$ is a probability density function we have $\int_{I_\alpha} p_\alpha = 1$. So this would require an unbounded $K_\alpha$; a case that we mostly exclude. We will discuss this in more detail below.

In Statement\,2, some relaxation is possible if it only needs to hold for finite $\lambda$, e.g.\ because $\sigma_0 < \infty$. Then it is sufficient, if the limit of
$-\tfrac{(p_\alpha K_\alpha)'}{2 p_\alpha K_\alpha} \pm \sqrt{\big(\tfrac{(p_\alpha K_\alpha)'}{2 p_\alpha K_\alpha}\big)^2 - \tfrac{\lambda}{K_\alpha}}$ is real and exists for $\pm$.

We highlighted the case of $\tfrac{p_\alpha'}{p_\alpha}$ being infinite at $a$ because we will later show that this is precisely the criterion for a singular SFA problem to have a fully discrete spectrum, i.e.\ $\sigma_0 = \infty$. By Corollary\,\ref{cor_SFA_spectrum_discrete}, if $a$ is infinite and $K_\alpha$ and $K'_\alpha$ are bounded, this criterion is a necessary one. So in this case, Statement\,2 holds generally for $\sigma_0 = \infty$.
\end{rmk}
In the following proofs we make use of \textsl{Barb\u{a}lat's Lemma} \cite{barbalat1959systemes}\footnote{For a direct proof and as an easily accessible reference, we refer the reader to \cite{10.4169/amer.math.monthly.123.8.825}.}. That lemma formalizes the intuition that if an improper integral to infinity of a function is finite, the function must tend to zero at infinity. Essentially it states that this conclusion is only valid if the function is uniformly continuous. Intuitively, this excludes the possibility that the function may have an infinite amount of sharper and sharper peaks of sufficiently quickly decreasing measure. Notably, Barb\u{a}lat's Lemma implies that if the limit of the integrand exists at all, then it must be zero.
\begin{proof}[Proof of Lemma\,\ref{lem_limits_SFA}]
\textsl{Proof of Statement 1:}
By Theorem\,\ref{thm_zeros_singular_SFA}, $g_\alpha$ has a finite number of zeros and since $\lambda > 0$ it has at least one zero. Let $s_0$ with ${g_\alpha(s_0) = 0}$ denote the zero that is closest to $a$. Since $a$ is a boundary point and $g_\alpha$ obeys the boundary condition \eqref{neumann-s}, by Theorem\,\ref{thm_sfa_interlace} Property\,3 $g_\alpha$ cannot vanish at $a$, so we have $s_0 \neq a$. By Theorem\,\ref{thm_sfa_interlace} there is no stationary point between $a$ and $s_0$, so $g_\alpha$ must be monotonic in a neighborhood of $a$. We conclude that the limit $\lim\limits_{s_\alpha \to a} g_\alpha$ exists and is non-zero, possibly infinite though.

\textsl{Proof of Statement 2:}
If $g_\alpha'$ is finite at $a$, the limit of $\tfrac{g_\alpha'}{g_\alpha}$ exists and is finite because by Statement\,1 the limit of $g_\alpha$ exists and is non-zero. The same applies if $g_\alpha$ is finite at $a$ (and thus non-zero as explained above in the proof of Statement\,1) since by assumption the limit of $g_\alpha'$ exists at $a$. From now on we consider the case that $g_\alpha$ and $g_\alpha'$ are both infinite at $a$.
We rearrange \eqref{dgl-SL} as
\begin{equation} \label{dgl-SL2}
- \lambda p_\alpha \quad = \quad (p_\alpha K_\alpha)' \; \frac{g_\alpha'}{g_\alpha} \; + \; p_\alpha K_\alpha \; \frac{g_\alpha'}{g_\alpha} \; \frac{g_\alpha''}{g_\alpha'}
\end{equation}
Now $\tfrac{g_\alpha'}{g_\alpha}$ is an indefinite form at $a$ and by L'Hôpital's rule we can instead study the behavior of $\tfrac{g_\alpha''}{g_\alpha'}$ at $a$ (note that the identity ${\tfrac{g_\alpha'}{g_\alpha} = \tfrac{g_\alpha''}{g_\alpha'}}$ only holds if the right hand side ${\tfrac{g_\alpha''}{g_\alpha'}}$ exists, which is yet to be shown). Based on this we further arrange \eqref{dgl-SL2} as
\begin{equation} \label{dgl-SL3}
0 \quad = \quad \bigg(\frac{g_\alpha''}{g_\alpha'}\bigg)^2 \, + \; \frac{(p_\alpha K_\alpha)'}{p_\alpha K_\alpha} \; \frac{g_\alpha''}{g_\alpha'} \; + \; \frac{\lambda}{K_\alpha}
\end{equation}
and solve for $\tfrac{g_\alpha''}{g_\alpha'}$:
\begin{equation} \label{dgl-SL3_solve}
\frac{g_\alpha''}{g_\alpha'} \quad = \quad - \, \frac{(p_\alpha K_\alpha)'}{2 p_\alpha K_\alpha} \; \pm \; \sqrt{\bigg(\frac{(p_\alpha K_\alpha)'}{2 p_\alpha K_\alpha}\bigg)^2 \; - \; \frac{\lambda}{K_\alpha}}
\end{equation}
Since in general we have to assume $\lambda$ to be unbounded (as noted above, in presence of an LP endpoint it may occur that $\sigma_0 < \infty$, i.e.\ the discrete spectrum may be bounded), \eqref{dgl-SL3_solve} can only turn out real for every $\lambda$ if $\tfrac{(p_\alpha K_\alpha)'}{p_\alpha K_\alpha}$ is infinite at $a$, given that $K_\alpha$ is non-zero. Then, depending on the signs, \eqref{dgl-SL3_solve} can yield the limit of $\tfrac{g_\alpha''}{g_\alpha'}$ as $0$, $\infty$ or $- \infty$. In either case it exists, so by L'Hôpital's rule the limit of $\tfrac{g_\alpha'}{g_\alpha}$ exists as well and has the same value.

\textsl{Proof of Statement 3:}
If $\lim\limits_{s_\alpha \to a} p_\alpha K_\alpha$ does not exist or is non-zero, by \eqref{neumann-s}, $g_\alpha'$ is zero at $a$. The limit of $\tfrac{g_\alpha'}{g_\alpha}$ exists and is zero, because we know that as an SFA solution, $g_\alpha$ is non-zero at the boundary and its limit exists by Statement\,1.
\end{proof}

In the following lemma we give conditions that justify the indeterminate term of equation \eqref{S16_A27}.
\begin{lem}[justification of equation \eqref{S16_A27} for the singular case] \label{lem_just_S16_A27}
Assuming that $p_\alpha$ and $K_\alpha$ obey the assumptions of Theorem\,\ref{thm_sfa_interlace},
let $g_\alpha(s_\alpha)$ denote a solution component of a possibly singular SFA problem for a discrete eigenvalue $\lambda < \sigma_0$ on the domain $s_\alpha \in I_\alpha$. Let $a$ denote a possibly singular endpoint of $I_\alpha$. If $a$ is infinite we require that the limit of
$p_\alpha g_\alpha^2$ exists at $a$ (e.g.\ by Lemma\,\ref{lem_pgg_exists} below)
and $K_\alpha$ is bounded at $a$.
Then it holds that
\begin{equation} \label{S16_A27_core}
\lim\limits_{s_\alpha \to \; a} \quad g_\alpha(s_\alpha) \; p_\alpha(s_\alpha) \; K_\alpha (s_\alpha) \; g'_\alpha(s_\alpha) \quad = \quad 0
\end{equation}
\end{lem}

\begin{proof}[Proof of Lemma\,\ref{lem_just_S16_A27}]
By Lemma\,\ref{lem_limits_SFA}, Statement\,1, the limit of $g_\alpha$ exists at $a$, possibly infinite though.
If $g_\alpha$ is finite at $a$, equation \eqref{S16_A27_core} holds because according to the Neumann boundary condition \eqref{neumann-s} we have $\lim\limits_{s_\alpha \to a} p_\alpha K_\alpha g'_\alpha=0$. Notably, $g_\alpha$ is finite at $a$ if $a$ is regular. So, for the remaining proof we assume that $a$ is a singular endpoint and $g_\alpha$ is infinite at $a$.

We first consider the case that $a$ is finite. In that case, there certainly exists a proper probability density function $p_0^+$ on $I_\alpha$ that is strictly positive at $a$. Similarly, a proper $K_0^+$ exists such that w.r.t.\ $p_0^+$ and $K_0^+$, $a$ would be a regular endpoint.
We embed our original $p_\alpha$ and $K_\alpha$ into function families for $t \in [0, 1]$:
\begin{equation}
p_t^+ \quad \coloneq \quad (1-t) p_0^+ \; + \; t p_\alpha, \qquad
K_t^+ \quad \coloneq \quad (1-t) K_0^+ \; + \; t K_\alpha
\end{equation}
Let $g_t^+$ denote the SFA solution w.r.t.\ $p_t^+$, $K_t^+$.
Since $a$ is regular except for $t = 1$ we have $g_t^+$ finite for $t \in [0, 1)$ and as explained above:
\begin{equation} \label{limits1}
\forall \; t \in [0, 1) \colon
\quad \lim\limits_{s_\alpha \to \; a} \quad g_t^+ \; p_t^+ \; K_\alpha \; g_t^+\vphantom{g_t}' \quad = \quad 0
\end{equation}
By continuity, \eqref{limits1} also holds in the limit $t \to 1$.
It is well known, e.g.\ \cite{Burke_ODE_Continuity}, that if the coefficients of an ordinary differential equation converge uniformly, then also the solutions and their derivatives converge uniformly, i.e.\ we have uniform convergence of $\lim\limits_{t \to 1} g_t^+ = g_\alpha$ and $\lim\limits_{t \to 1} g_t^+\vphantom{g_t}' = g'_\alpha$. This permits switching the order of the limit processes in \eqref{limits1} regarding $s_\alpha$ and $t$, completing the proof for finite $a$:
\begin{equation} \label{limits3}
\lim\limits_{s_\alpha \to \; a} \quad \underbrace{\lim\limits_{t \to 1} \quad g_t^+ \; p_t^+ \; K_\alpha \; g_t^+\vphantom{g_t}'}_{\qquad \quad \;\;\; = \,g_\alpha p_\alpha K_\alpha g'_\alpha} \quad = \quad 0
\end{equation}

For the remaining proof we assume $\abs{a} = \infty$. Since we assume $\lim\limits_{s_\alpha \to a} p_\alpha g_\alpha^2$ exists and given the SFA constraint $\int_{I_\alpha} p_\alpha g_\alpha^2 = 1$ we conclude by Barb\u{a}lat's Lemma
\begin{equation} \label{pgg_zero}
\lim\limits_{s_\alpha \to a} \; p_\alpha g_\alpha^2 \quad = \quad 0
\end{equation}
We distinguish the cases that $\tfrac{g_\alpha'}{g_\alpha}$ is at $a$
(i) bounded but possibly non-existing, (ii) existing and infinite, (iii) unbounded and non-existing.

\textsl{Case (i), $\tfrac{g_\alpha'}{g_\alpha}$ bounded at $a$, possibly non-existing:}
\begin{equation}
\lim\limits_{s_\alpha \to \; a} \; g_\alpha p_\alpha K_\alpha g'_\alpha \quad = \quad \lim\limits_{s_\alpha \to \; a} \; \underbrace{g_\alpha^2 p_\alpha}_{\substack{\to 0 \\ \text{\tiny{\eqref{pgg_zero}}}}} K_\alpha \underbrace{\frac{g'_\alpha}{g_\alpha}}_{< \infty} \quad = \quad 0
\end{equation}

\textsl{Case (ii), $\tfrac{g_\alpha'}{g_\alpha}$ exists and is infinite at $a$:}
As we assume $g_\alpha$ infinite, we have $\tfrac{1}{g_\alpha}$ tending to zero at $a$ and 
by \eqref{neumann-s} also $\lim\limits_{s_\alpha \to a} p_\alpha K_\alpha g'_\alpha=0$. Hence we can apply L'Hôpital's rule:
\begin{align}
\lim\limits_{s_\alpha \to \; a} \; g_\alpha p_\alpha K_\alpha g'_\alpha
\quad &= \quad
\lim\limits_{s_\alpha \to \; a} \; \frac{p_\alpha K_\alpha g'_\alpha}{\frac{1}{g_\alpha}}
\quad \eq_{\text{L'H}} \quad
\lim\limits_{s_\alpha \to \; a} \; \frac{(p_\alpha K_\alpha g'_\alpha)'}{-\frac{g_\alpha'}{g_\alpha^2}} \quad \eq_{\eqref{dgl-SL}} \quad
\lim\limits_{s_\alpha \to \; a} \; \frac{\lambda_\alpha p_\alpha g_\alpha}{\frac{g_\alpha'}{g_\alpha^2}} \\
&= \quad
\lambda_\alpha \; \lim\limits_{s_\alpha \to \; a} \; \underbrace{p_\alpha g_\alpha^2}_{\substack{\to 0 \\ \text{\tiny{\eqref{pgg_zero}}}}} \underbrace{\frac{g_\alpha}{g'_\alpha}}_{\to 0} \quad = \quad 0
\end{align}

\textsl{Case (iii), $\tfrac{g_\alpha'}{g_\alpha}$ is unbounded and non-existing at $a$:}
First note that by Theorem\,\ref{thm_zeros_singular_SFA} the number of zeros of $g_\alpha$ is finite since we assume that $\lambda$ is finite. Hence, by Theorem\,\ref{thm_sfa_interlace} the number of stationary points of $g_\alpha$ must be finite as well. So, $g_\alpha'$ must be non-zero in some neighborhood of $a$ and also at $a$ since otherwise we were in case (i) or (ii). That means, $\tfrac{g_\alpha'}{g_\alpha}$ cannot oscillate around zero but just within strictly positive or strictly negative values. This is crucial for the applicability of L'Hôpital's rule because that rule requires that the derivative of the denominator function is strictly non-zero in some neighborhood of the limit point. This is likewise required for the following generalization of L'Hôpital's rule for $\liminf$ and $\limsup$, which we utilize for proving case (iii):
\begin{equation} \label{limsup_LH}
\liminf_{x \to a} \; \frac{f'(x)}{g'(x)} \quad \leq \quad
\liminf_{x \to a} \; \frac{f(x)}{g(x)} \quad \leq \quad
\limsup_{x \to a} \; \frac{f(x)}{g(x)} \quad \leq \quad
\limsup_{x \to a} \; \frac{f'(x)}{g'(x)}
\end{equation}
Since $\tfrac{g_\alpha'}{g_\alpha}$ is non-zero near $a$ and at $a$, it follows that $\abs{\tfrac{g_\alpha}{g_\alpha'}} < \infty$ near $a$ and at $a$. Note that also $\tfrac{g_\alpha'}{g_\alpha^2}$ is non-zero near $a$ and at $a$, which is relevant in the following calculation. We distinguish the cases that $g_\alpha g_\alpha'$ is positive or negative near $a$ and at $a$.

\textsl{Case $g_\alpha g_\alpha' > 0$ near $a$ and at $a$:}
\begin{align}
0 \quad \dleq_{\hphantom{\text{L'H}}}& \quad \limsup\limits_{s_\alpha \to \; a} \; g_\alpha p_\alpha K_\alpha g'_\alpha
\quad = \quad \limsup \limits_{s_\alpha \to \; a} \; \frac{p_\alpha K_\alpha g'_\alpha}{\frac{1}{g_\alpha}} \quad
\dleq_{\substack{\text{L'H} \\ \eqref{limsup_LH}}} \quad
\limsup \limits_{s_\alpha \to \; a} \; \frac{(p_\alpha K_\alpha g'_\alpha)'}{-\frac{g_\alpha'}{g_\alpha^2}} \\
\eq_{\substack{\eqref{dgl-SL} \\ \hphantom{\text{L'H}}}}& \quad
\limsup \limits_{s_\alpha \to \; a} \; \frac{\lambda p_\alpha g_\alpha}{\frac{g_\alpha'}{g_\alpha^2}} \quad
\eq_{\hphantom{\text{L'H}}} \quad
\limsup \limits_{s_\alpha \to \; a} \; \underbrace{p_\alpha g_\alpha^2}_{\substack{\to 0 \\ \text{\tiny{\eqref{pgg_zero}}}}} \; \lambda \; \underbrace{\frac{ g_\alpha}{g_\alpha'}}_{< \infty}
\quad = \quad 0
\end{align}

\textsl{Case $g_\alpha g_\alpha' < 0$ near $a$ and at $a$:}
\begin{align}
0 \quad \dgeq_{\hphantom{\text{L'H}}}& \quad \liminf\limits_{s_\alpha \to \; a} \; g_\alpha p_\alpha K_\alpha g'_\alpha
\quad = \quad \liminf \limits_{s_\alpha \to \; a} \; \frac{p_\alpha K_\alpha g'_\alpha}{\frac{1}{g_\alpha}} \quad
\dgeq_{\substack{\text{L'H} \\ \eqref{limsup_LH}}} \quad
\liminf \limits_{s_\alpha \to \; a} \; \frac{(p_\alpha K_\alpha g'_\alpha)'}{-\frac{g_\alpha'}{g_\alpha^2}} \\
\eq_{\substack{\eqref{dgl-SL} \\ \hphantom{\text{L'H}}}}& \quad
\liminf \limits_{s_\alpha \to \; a} \; \frac{\lambda p_\alpha g_\alpha}{\frac{g_\alpha'}{g_\alpha^2}} \quad
\eq_{\hphantom{\text{L'H}}} \quad
\liminf \limits_{s_\alpha \to \; a} \; \underbrace{p_\alpha g_\alpha^2}_{\substack{\to 0 \\ \text{\tiny{\eqref{pgg_zero}}}}} \; \lambda \; \underbrace{\frac{ g_\alpha}{g_\alpha'}}_{> -\infty}
\quad = \quad 0
\end{align}
These results can be combined to the equation
\begin{equation}
0 \quad \leq \quad \limsup\limits_{s_\alpha \to \; a} \; \abs{g_\alpha p_\alpha K_\alpha g'_\alpha}
\quad \leq \quad 0
\end{equation}
and show that the limit of $g_\alpha p_\alpha K_\alpha g'_\alpha$ exists at $a$ and is zero.
\end{proof}

\begin{corr}[justification of SFA theory for the singular case] \label{cor_sfa_singular_justify}
If the assumptions of Lemma\,\ref{lem_just_S16_A27} apply, the following properties of an SFA problem also hold in the singular case:
\begin{itemize}
\item Equation (S16) from supplementary material of \cite{FranziusSprekelerEtAl-2007e}, or similarly (A27) from \cite{sprekeler2009slowness} hold.
\item The delta value of the $i$th SFA solution component equals the $i$th eigenvalue $\lambda_i$. This still holds for non-zero $q$ with the $i$th delta value denoting ${\int_a^b \; g_{\alpha i}'^2 \; K_\alpha \; p_\alpha \; + \; g_{\alpha i}^2 \; q \; d s_\alpha}$ from \eqref{SL_opt_problem}.
\item SFA eigenvalues and delta values are positive and finite.
\item Subsequent results from SFA theory like \eqref{A_dgl-x} and \eqref{A_dgl-x-decomp} in \ref{sec:A_sfa}.
\end{itemize}
\end{corr}

\begin{proof}[Proof of Corollary\,\ref{cor_sfa_singular_justify}]
Equation \eqref{S16_A27_core} was the only bit in (S16) from supplementary material of \cite{FranziusSprekelerEtAl-2007e}, or (A27) from \cite{sprekeler2009slowness} that required justification. This is now provided by Lemma\,\ref{lem_just_S16_A27}. Together with the self-adjointness of SFA problems, provided by Lemma\,\ref{lem_singular_SFA_self_adjoint}, this justifies the original proofs from \cite{FranziusSprekelerEtAl-2007e, sprekeler2009slowness} for the statements above.
\end{proof}

In Lemma\,\ref{lem_just_S16_A27} we assume that the limit of $p_\alpha g_\alpha^2$ exists at a singular boundary point. In the following lemma we provide a criterion in terms of the coefficient functions to justify this assumption.

\begin{lem}[existence of $p_\alpha g_\alpha^2$ at a singular boundary point] \label{lem_pgg_exists}
Let $g_\alpha(s_\alpha)$ denote a solution component of a possibly singular SFA problem for a discrete eigenvalue $\lambda < \sigma_0$ on the domain $s_\alpha \in I_\alpha$. Let $a$ denote a possibly singular endpoint of $I_\alpha$. If $g_\alpha$ has a finite delta value, i.e.\ evaluation of \eqref{thm_1D_SFA_value_func} is finite for $g_\alpha$, and
\begin{equation} \label{lem_pgg_exists_eq1}
\lim\limits_{s_\alpha \to a} \quad \frac{K_\alpha \abs{p_\alpha'}}{p_\alpha} \quad = \quad \infty \hphantom{0}
\end{equation}
then it holds that
\begin{equation} \label{lem_pgg_exists_eq2}
\lim\limits_{s_\alpha \to a} \quad\;\; p_\alpha g_\alpha^2 \quad = \quad 0 \hphantom{\infty}
\end{equation}
\end{lem}
\begin{rmk}
If $K_\alpha$ is non-zero at $a$, then $\tfrac{\abs{p_\alpha'}}{p_\alpha} \Lra\limits_{s_\alpha \to a} \infty$ is sufficient for \eqref{lem_pgg_exists_eq1} to hold. In the next section we will see that this condition is at the same time sufficient to guarantee a discrete spectrum.

In \cite{SprekelerZitoEtAl-2014} it is pointed out that SFA solutions are sought in the ``Sobolev space of functions for which both the functions themselves as well as all their partial derivatives w.r.t.\ the input signals are square-integrable w.r.t\ the probability measure of the input signals'' (first footnote in Section\,3.1). This asserts that for SFA, only solutions are considered for which \eqref{thm_1D_SFA_value_func} is finite.
\end{rmk}

\begin{proof}[Proof of Lemma\,\ref{lem_pgg_exists}]
We first show that the term $g_\alpha p_\alpha K_\alpha g_\alpha'$ is bounded at $a$.
By Definition\,\ref{SL-problem}, we have $g_\alpha, \; p_\alpha K_\alpha g_\alpha' \in {AC}_\text{loc}(a, b)$ since $r = p_\alpha K_\alpha$. As explained before Definition\,\ref{SL-problem}, this implies that $g_\alpha$ and $p_\alpha K_\alpha g_\alpha'$ are bounded everywhere on $(a, b)$ except possibly at the boundary points $a$ or $b$. Consequently, the same applies to their product $g_\alpha p_\alpha K_\alpha g_\alpha'$. By Theorem\,\ref{thm_zeros_singular_SFA} we know that $g_\alpha$ has at least one zero on $(a, b)$, so there exists a point $a_0 \in (a, b)$ where $g_\alpha p_\alpha K_\alpha g_\alpha'$ vanishes. The SFA delta value ${\int_a^b g_\alpha'^2 K_\alpha p_\alpha \; d s_\alpha}$ is assumed to be finite and its integrand is strictly positive, so we have
\begin{equation} \label{SFA_finite_delta}
\infty \quad > \quad \int_a^b \; g_\alpha'^2 \; K_\alpha \; p_\alpha \; d s_\alpha \quad > \quad \int_a^{a_0} \; g_\alpha'^2 \; K_\alpha \; p_\alpha \; d s_\alpha
\end{equation}
We apply integration by parts, similarly to \eqref{S16_A27}:
\begin{equation} \label{SFA_finite_delta_int_by_parts}
0 \quad < \quad \int_a^{a_0} \; g_\alpha'^2 \; K_\alpha \; p_\alpha \; d s_\alpha \quad = \quad
\underbrace{\Big[ \underbrace{g_\alpha' \; K_\alpha \; p_\alpha \; g_\alpha}_{= \; 0 \; \text{at} \; a_0} \; \Big]_a^{a_0}}_{= \; -g_\alpha' K_\alpha p_\alpha g_\alpha \big|_a}
\quad - \quad
\int_a^{a_0} g_\alpha \; \underbrace{\big(g_\alpha' \; K_\alpha \; p_\alpha \big)'}_{\eq_\eqref{dgl-SL} -\lambda p_\alpha g_\alpha} \; d s_\alpha
\end{equation}
Note that in this proof it is not assumed that the eigenvalue $\lambda$ equals the delta value \eqref{SFA_finite_delta}. We just assume that both values are finite.
From \eqref{SFA_finite_delta} and \eqref{SFA_finite_delta_int_by_parts} we conclude
\begin{equation} \label{SFA_finite_delta2}
\infty \quad > \quad -g_\alpha' K_\alpha p_\alpha g_\alpha \big|_a + \lambda \underbrace{\int_a^{a_0} g^2_\alpha \; p_\alpha \; d s_\alpha}_{\quad\! 0 \; < \, \cdot \, <_\eqref{thm_1D_SFA_constr_decor} \, 1} \quad > \quad 0 \qquad \Rightarrow \qquad \Bigabs{g_\alpha' K_\alpha p_\alpha g_\alpha \big|_a} \quad < \quad \infty
\end{equation}
By the generalized L'Hôpital's rule \eqref{limsup_LH} we have
\begin{align}
0 \quad \leq& \quad \limsup\limits_{s_\alpha \to a} \; p_\alpha g_\alpha^2 \quad = \quad \limsup\limits_{s_\alpha \to a} \; \frac{g_\alpha^2}{\frac{1}{p_\alpha}} \quad \dleq_{\text{L'H}} \quad \limsup\limits_{s_\alpha \to a} \; \frac{2g_\alpha g_\alpha'}{-\frac{p_\alpha'}{p_\alpha^2}} \\
\eq& \quad \limsup\limits_{s_\alpha \to a} \; \underbrace{g_\alpha p_\alpha K_\alpha g_\alpha'}_{\abs{\,\cdot\,} \; <_\eqref{SFA_finite_delta2} \; \infty} \; \underbrace{\frac{-2 p_\alpha}{K_\alpha p_\alpha'}}_{\to_{\eqref{lem_pgg_exists_eq1}} 0}
\quad = \quad 0
\end{align}
\end{proof}

\section{Discreteness of the Spectrum} \label{sec:discrete_spectrum}
A famous historical result for singular Sturm-Liouville problems is due to \cite{molchanov1953}. A sufficient but not necessary criterion for a Sturm-Liouville problem having a BD spectrum (discrete and bounded below) in case $r = w = 1$ is that
\begin{equation} \label{molchanov}
\lim\limits_{t \to \infty} \; \int_t^{t+c} q(x) \; dx \quad = \quad \infty \qquad \forall \; c > 0
\end{equation}
Since for SFA problems we have $q = 0$, criterion \eqref{molchanov} is inherently false for SFA problems. Taking into account that the criterion is not necessary, it is basically meaningless for SFA. The constraint $r = w = 1$ is seen as a normalization by many authors, since one can bring every Sturm-Liouville problem into that form via transformations of the dependent and independent variables, \cite{courant1989methoden} Chapter\,V, §3.3. However, \cite{zettl2010sturm} points out that it is questionable whether such a transformation preserves the desired properties of the spectrum (discussion in Chapter\,10.13). \cite{KWONG198153} generalize Molchanov's criterion to other choices of $w$ and $r$ and in other ways, but all their results (theorems and corollaries in Section\,2 of \cite{KWONG198153}) are orthogonal to the $q = 0$ case. A suitable criterion for the $q=0$ case was pointed out by Christian Remling in \cite{356592}:
\begin{thm}[discreteness of the spectrum of singular Sturm-Liouville problems with $q=0$] \label{thm_SL_spectrum_discrete}
(from \cite{356592})

The spectrum of a (possibly singular) Sturm-Liouville problem with $q=0$ is discrete if and only if $w \in L^1$ and for every singular endpoint $c$ of the domain it holds that
\begin{equation} \label{thm_SL_spectrum_discrete_eq}
\lim_{x \to c} \; \int_{x_0}^x w v^2\; dt \; \int_x^{c} w\; dt \quad = \quad 0
\end{equation}
where $x_0$ is the regular endpoint of the domain or an arbitrary point from the domain if both endpoints are singular, $v$ is the same as in \eqref{lambda0_solutions} with $r = K_\alpha p_\alpha$, i.e.\ $v(x) = \int_0^x \tfrac{1}{r} \; dt$ and $r, w, q$ are those from \eqref{SL-diff_op}.
\end{thm}
\begin{rmk1}
Note that by Corollary\,\ref{cor_sfa_singular_justify}, all eigenvalues of a possibly singular SFA problem are positive, so if the spectrum is discrete, it is actually BD.
\end{rmk1}

\begin{proof}[Proof of Theorem\,\ref{thm_SL_spectrum_discrete}]
At this place, we work out the proof sketched in \cite{356592} in detail. It works by translating the Sturm-Liouville problem to a \textsl{canonical system} and applying the main result from \cite{ROMANOV2020108318}. It should be noted that the result is only formulated for the half line, i.e.\ domains with just one singular endpoint. However, by \cite{naimark1968linear}, §24, Theorem\,1, such results can be applied to both half lines separately to cover the full line case.
For the transformation to a canonical system we proceed according to \cite{REMLING2020105395}, Section\,5 -- ``An Example''.
The $\lambda = 0$ solutions from \eqref{lambda0_solutions} are arranged in the \textsl{fundamental matrix} $\vf{T_0}$ and $Y$ as
\begin{equation}
\vf{T_0} \quad \coloneq \quad \begin{pmatrix} r u' & r v' \\ u & v \end{pmatrix}
\quad \eq_{\eqref{lambda0_solutions}} \quad \begin{pmatrix} 0 & 1 \\ 1 & v \end{pmatrix}
, \qquad Y \quad \coloneq \quad \begin{pmatrix} r y' \\ y \end{pmatrix}
\end{equation}
The \textsl{Hamiltonian} $\vf{H}$ and the \textsl{symplectic matrix} $\vf{J}$ are given as
\begin{equation} \label{def_Hamil_sympl}
\vf{H} \quad \coloneq \quad w \begin{pmatrix} u^2 & uv \\ uv & v^2 \end{pmatrix}
\quad \eq_{\eqref{lambda0_solutions}} \quad w \begin{pmatrix} 1 & v \\ v & v^2 \end{pmatrix}
, \qquad \vf{J} \quad \coloneq \quad \begin{pmatrix} 0 & -1 \\ 1 & 0 \end{pmatrix}
\end{equation}
One verifies that $y$ is a solution of \eqref{SL-problem_eq} if and only if $U$ with $Y = \vf{T_0} U$ solves the canonical system
\begin{equation}
\vf{J} \; U' \quad = \quad - \lambda \; \vf{H} \; U
\end{equation}
By \cite{ROMANOV2020108318}, Theorem\,1.1, the spectrum of a canonical system with Hamiltonian $\vf{H} = \left(\begin{smallmatrix}h_1 & h_3 \\ h_3 & h_2\end{smallmatrix}\right)$ on the half line $[a, b)$ with singular $b$ and $\int_a^b h_1 < \infty$ is discrete if and only if
\begin{equation} \label{crit_cansys_discr_spec}
\lim_{x \to b} \; \int_x^b h_1(s) \; ds \; \int_a^x h_2(s) \; ds \quad = \quad 0
\end{equation}
With the Hamiltonian from \eqref{def_Hamil_sympl} this turns out as \eqref{thm_SL_spectrum_discrete_eq} for $a = x_0$, $b = c$. Note that the requirement of $h_1 = w$ having a finite integral is satisfied because for \eqref{thm_SL_spectrum_discrete_eq} we require $w \in L^1$. However, in \cite{ROMANOV2020108318} they state that this requirement is not essential and is merely a normalization without loss of generality.
In case of two singular endpoints, we can choose $x_0$ from the inner domain and apply \eqref{crit_cansys_discr_spec} by \cite{naimark1968linear}, §24, Theorem\,1 for each singular end point $c$ by means of $a = x_0$, $b = c$.
\end{proof}

\begin{thm}[criterion for singular SFA problems to have a discrete spectrum] \label{thm_SFA_spectrum_discrete}
Assume that the limit of $p_\alpha$ exists at every singular endpoint of the domain $I_\alpha$. A sufficient criterion for a singular SFA problem to have a discrete spectrum is that for every singular endpoint $c$
\begin{equation}  \label{thm_SFA_spectrum_discrete_eq}
\lim\limits_{s_\alpha \to c} \;\; \Abs{\left(\sqrt{K_\alpha}\,\right)' \; + \; \sqrt{K_\alpha} \; \frac{p_\alpha'}{p_\alpha}} \quad = \quad \infty
\end{equation}

If (but not only if) the following conditions hold
\begin{align} 
\int_{x_0}^c \; p_\alpha \; \bigg(\int_{x_0}^{t} \; \frac{1}{K_\alpha p_\alpha} \; ds \bigg)^2 \; dt \quad &= \quad \infty \label{thm_SFA_spectrum_discrete_necf} \\
\lim\limits_{s_\alpha \to c} \; \; \sqrt{K_\alpha} \; p_\alpha \quad &= \quad 0
\label{thm_SFA_spectrum_discrete_nec3}
\end{align}
where $x_0$ can be chosen like described in Theorem\,\ref{thm_SL_spectrum_discrete},
then criterion \eqref{thm_SFA_spectrum_discrete_eq} is necessary.

For the special cases $c=\pm \infty$, condition \eqref{thm_SFA_spectrum_discrete_necf} can be replaced by the conditions
\begin{align}
&\lim\limits_{s_\alpha \to c} \; \; K_\alpha \, p_\alpha \quad < \quad \infty \label{thm_SFA_spectrum_discrete_nec1} \\
&\lim\limits_{s_\alpha \to c} \;\; \Abs{K_\alpha \; \left(\sqrt{p_\alpha}\,\right)'} \quad < \quad \infty
\label{thm_SFA_spectrum_discrete_nec2}
\end{align}
which must then both hold in addition to \eqref{thm_SFA_spectrum_discrete_nec3} for criterion \eqref{thm_SFA_spectrum_discrete_eq} to be necessary.
In particular, \eqref{thm_SFA_spectrum_discrete_nec1}, \eqref{thm_SFA_spectrum_discrete_nec2} and \eqref{thm_SFA_spectrum_discrete_nec3} hold if $K_\alpha$ is bounded and the limits of $p_\alpha$ and $p_\alpha'$ exist at $c$.
\end{thm}
Note that $K_\alpha$ or $K'_\alpha$ being unbounded are odd conditions for an SFA problem. $K_\alpha$ arises from movement velocity statistics of an agent. Unbounded velocity or acceleration are unphysical concepts (on the other hand, an unbounded domain may be an unphysical concept as well). So practically, one can mostly assume that \eqref{thm_SFA_spectrum_discrete_eq} is a necessary criterion for open-space scenarios, i.e.\ with singular endpoint(s) at infinity. Assuming we can exclude various odd conditions, the following corollary yields the simple discreteness criterion we announced in the introduction:

\begin{corr}[simple criterion for a singular SFA problem to have discrete spectrum] \label{cor_SFA_spectrum_discrete}
Assume that at every singular endpoint of the domain $I_\alpha$: $K_\alpha$ is non-zero, $K_\alpha$ and $K'_\alpha$ are bounded and the limits of $p_\alpha$ and $p_\alpha'$ exist.
Then a sufficient criterion for a singular SFA problem to have a discrete spectrum is that for every singular endpoint $c$
\begin{equation} \label{cor_SFA_spectrum_discrete_eq}
\lim\limits_{s_\alpha \to c} \; \frac{\abs{p_\alpha'}}{p_\alpha} \quad = \quad \infty
\end{equation}
If $c$ is infinite, then the criterion is also necessary.
\end{corr}

One might wonder if condition \eqref{cor_SFA_spectrum_discrete_eq} is not actually fulfilled by \textsl{every} probability density function on an unbounded (or half bounded) domain. An easy counter example is given by the probability density function
\begin{equation}
p_{\alpha \epsilon}(s_\alpha) \quad \coloneq \quad \frac{\epsilon}{s_\alpha^{1+\epsilon}}
\end{equation}
on $I_\alpha = [1, \infty)$ for every $\epsilon > 0$. \eqref{cor_SFA_spectrum_discrete_eq} in particular admits exponential-based probability density functions like $p_\alpha(s_\alpha) = \exp(-\tfrac{s_\alpha^2}{2})$ for the normal distribution. This reflects the fact that ${\frac{p_\alpha'}{p_\alpha} = (\log p_\alpha)'}$, so exponentials are a good fit for \eqref{cor_SFA_spectrum_discrete_eq}. Before we move on to the longer proof of Theorem\,\ref{thm_SFA_spectrum_discrete}, we prove Corollary\,\ref{cor_SFA_spectrum_discrete} in just a few lines by using Theorem\,\ref{thm_SFA_spectrum_discrete}.

\begin{proof}[Proof of Corollary\,\ref{cor_SFA_spectrum_discrete}]
As discussed earlier, since the limit of $p_\alpha$ exists at $c$ it must be zero since its integral is finite. The same applies to $p_\alpha'$. Then, given $K_\alpha$ is bounded,
\eqref{thm_SFA_spectrum_discrete_nec1}, \eqref{thm_SFA_spectrum_discrete_nec2} and \eqref{thm_SFA_spectrum_discrete_nec3} hold, so \eqref{thm_SFA_spectrum_discrete_eq} is sufficient and necessary. Since $K_\alpha$ and $K'_\alpha$ are bounded and $K_\alpha$ is non-zero, \eqref{thm_SFA_spectrum_discrete_eq} is dominated by $\tfrac{p_\alpha'}{p_\alpha}$ and holds if and only if \eqref{cor_SFA_spectrum_discrete_eq} holds.
\end{proof}

\begin{proof}[Proof of Theorem\,\ref{thm_SFA_spectrum_discrete}]
We apply Theorem\,\ref{thm_SL_spectrum_discrete} to the SFA setting. Equation \eqref{thm_SL_spectrum_discrete_eq} turns out to be
\begin{equation} \label{thm_SFA_spectrum_discrete_eq2}
\lim_{s_\alpha \to c} \; \int_{x_0}^{s_\alpha} p_\alpha v^2\; dt \; \int_{s_\alpha}^c p_\alpha \; dt \quad = \quad 0
\end{equation}
with
\begin{equation} \label{thm_SFA_spectrum_discrete_eq3}
v(s_\alpha) \quad = \quad \int_{x_0}^{s_\alpha} \; \frac{1}{K_\alpha p_\alpha} \; dt
\end{equation}
We denote the antiderivative of $p_\alpha$ as
\begin{equation} \label{p_antiderivative}
P_\alpha(s_\alpha) \quad \coloneq \quad \int_{x_0}^{s_\alpha} \; p_\alpha(t) \; dt, \qquad \lim\limits_{s_\alpha \to c} \; P_\alpha \quad \eqcolon \quad c_p
\end{equation}
The $p_\alpha$ integral from \eqref{thm_SFA_spectrum_discrete_eq2} then becomes:
\begin{equation} \label{p_integral}
\int_{s_\alpha}^c \; p_\alpha(t) \; dt \quad = \quad \left[P_\alpha\right]_{s_\alpha}^c \quad = \quad c_p - P_\alpha(s_\alpha) \quad \Lra_{s_\alpha \!\!\to\, c} \quad c_p-c_p \quad = \quad 0
\end{equation}
If \eqref{thm_SFA_spectrum_discrete_nec1} does not hold, the integrand of $v$ becomes zero at $c$ and $v$ might be bounded. In that case, \eqref{thm_SFA_spectrum_discrete_eq2} might hold without further requirements and \eqref{thm_SFA_spectrum_discrete_eq} may or may not be necessary. In case of $c=\infty$, note that $\tfrac{1}{K_\alpha p_\alpha}$ may be zero at $c$ and $v$ still infinite at $c$. This is precisely the reason why the theorem states ``If (but not only if).''
From now on we assume the case that either \eqref{thm_SFA_spectrum_discrete_necf} holds or that $c$ is infinite and \eqref{thm_SFA_spectrum_discrete_nec1} holds, i.e.\ $\tfrac{1}{K_\alpha p_\alpha }$ is non-zero at $c$. For infinite $c$, this means that \eqref{thm_SFA_spectrum_discrete_eq3} is infinite at $c$. For possibly finite $c$ observe that \eqref{thm_SFA_spectrum_discrete_necf} is only feasible if \eqref{thm_SFA_spectrum_discrete_eq3} is infinite at $c$, so we can assume this in either case.
We examine the integrand of $\int_{x_0}^{s_\alpha} p_\alpha v^2$ to decide whether it is bounded at an infinite $c$ (for a finite $c$ it is infinite due to \eqref{thm_SFA_spectrum_discrete_necf}):
\begin{align}
\lim\limits_{s_\alpha \to c} \; p_\alpha v^2 \quad &\eq_{\substack{\eqref{thm_SFA_spectrum_discrete_eq3} \\ \hphantom{\text{L'H}}}} \quad \lim\limits_{s_\alpha \to c} \; \underbrace{p_\alpha}_{\to 0} \; \underbrace{\left( \int_{x_0}^{s_\alpha} \; \frac{1}{K_\alpha p_\alpha} \; dt \right)^2}_{\to \infty} \quad = \quad \left( \lim\limits_{s_\alpha \to c} \;  \; \frac{\int_{x_0}^{s_\alpha} \; \frac{1}{K_\alpha p_\alpha} \; dt}{\frac{1}{\sqrt{p_\alpha}}} \right)^2 \\
&\eq_{\text{L'H}} \quad \left( \lim\limits_{s_\alpha \to c} \;  \; \frac{\frac{1}{K_\alpha p_\alpha}}{-\frac{p_\alpha'}{2 p_\alpha \sqrt{p_\alpha}}} \right)^2
\quad = \quad \left( \lim\limits_{s_\alpha \to c} \;  \; -\frac{1}{K_\alpha (\sqrt{p_\alpha} \,)'} \right)^2 \label{pvv} \quad \Lra_{\eqref{thm_SFA_spectrum_discrete_nec2}} \quad > 0
\end{align}
The result of equation \eqref{pvv} is non-zero if \eqref{thm_SFA_spectrum_discrete_nec2} holds. In that case, we have ${\lim\limits_{s_\alpha \to c} \int_{x_0}^{s_\alpha} p_\alpha v^2 dt = \infty}$.
Either by \eqref{thm_SFA_spectrum_discrete_necf} or for infinite $c$ by the calculus above, given \eqref{thm_SFA_spectrum_discrete_nec1}, \eqref{thm_SFA_spectrum_discrete_nec2}, we have ${\lim\limits_{s_\alpha \to c} \int_{x_0}^{s_\alpha} p_\alpha v^2 dt = \infty}$ and $\lim\limits_{s_\alpha \to c} \int_{x_0}^{s_\alpha} \frac{1}{K_\alpha p_\alpha} dt = \infty$.
From now on we will consider \eqref{thm_SFA_spectrum_discrete_eq2} as a term of the form $\infty \cdot 0$:
\begin{align}
\eqref{thm_SFA_spectrum_discrete_eq2} \quad &\eq_{\substack{\eqref{p_integral} \\ \hphantom{\text{L'H}}}} \quad
\lim_{s_\alpha \to c} \; \frac{\int_{x_0}^{s_\alpha} p_\alpha v^2\; dt}{\frac{1}{c_p - P_\alpha}} \quad \eq_{\substack{\text{L'H} \\ \eqref{p_antiderivative}}} \quad
\lim_{s_\alpha \to c} \; \frac{p_\alpha v^2}{\frac{p_\alpha}{(c_p - P_\alpha)^2}} \quad \eq_{\eqref{thm_SFA_spectrum_discrete_eq3}} \quad
\left( \lim_{s_\alpha \to c} \; \frac{\int_{x_0}^{s_\alpha} \; \frac{1}{K_\alpha p_\alpha} \; dt}{\frac{1}{c_p - P_\alpha}} \right)^2 \\
&\eq_{\substack{\text{L'H} \\ \eqref{p_antiderivative}}} \quad \left( \lim_{s_\alpha \to c} \; \frac{\frac{1}{K_\alpha p_\alpha}}{\frac{p_\alpha}{(c_p - P_\alpha)^2}} \right)^2 \quad = \quad
\Bigg( \lim_{s_\alpha \to c} \; \frac{c_p - P_\alpha}{\sqrt{K_\alpha}p_\alpha} \Bigg)^4
\quad \eq_{\text{L'H}} \quad 
\Bigg( \lim_{s_\alpha \to c} \; \frac{-p_\alpha}{\big(\sqrt{K_\alpha}p_\alpha\big)'} \Bigg)^4 \\
&\eq_{\hphantom{\text{L'H}}} \quad
\left( \lim_{s_\alpha \to c} \; \frac{-1}{\big(\sqrt{K_\alpha}\big)' \; + \; \sqrt{K_\alpha} \, \frac{p_\alpha'}{p_\alpha}} \right)^4
\quad \Lra_{\eqref{thm_SFA_spectrum_discrete_eq}} \quad 0
\end{align}
The last application of L'Hôpital's rule is required because of \eqref{thm_SFA_spectrum_discrete_nec3}. This is why \eqref{thm_SFA_spectrum_discrete_nec3} makes \eqref{thm_SFA_spectrum_discrete_eq} a necessary condition. We have shown that \eqref{thm_SFA_spectrum_discrete_eq} stands at the end of the L'Hôpital's rule sequence and forms the final case that must hold if all earlier terms turn out as indeterminate forms. The cases where such an earlier term may be determinate are precisely tracked by \eqref{thm_SFA_spectrum_discrete_nec1}, \eqref{thm_SFA_spectrum_discrete_nec2}, \eqref{thm_SFA_spectrum_discrete_nec3} for infinite $c$ or implicitly excluded by \eqref{thm_SFA_spectrum_discrete_necf} for possibly finite $c$.
\end{proof}

\section*{Summary}
We have investigated the singular setting of SFA, i.e.\ where the domain $I_\alpha$ of some source $s_\alpha$ is unbounded in at least one direction or where $p_\alpha$ or $K_\alpha$ becomes zero at the boundary. We have derived conditions solely in terms of the coefficient functions and their derivatives to guarantee that SFA solutions possess various profound properties known from the regular case. If the SFA delta value (weighted Dirichlet energy without factor $\tfrac{1}{2}$) is finite and at every singular endpoint $c$
\begin{itemize}
\item $K_\alpha$ is non-zero,
\item $K_\alpha$ and $K'_\alpha$ are bounded,
\item the limits of $p_\alpha$ and $p'_\alpha$ exist, and
\item $\lim\limits_{s_\alpha \to c} \; \frac{\abs{p_\alpha'}}{p_\alpha} \; = \; \infty$
\end{itemize}
then SFA theory is justified and the spectrum is discrete and bounded below (BD).
Given the first three conditions, the fourth one asserts for an SFA solution $g_\alpha$ that $p_\alpha g_\alpha^2$ vanishes at singular boundary points (Lemma\,\ref{lem_pgg_exists}) and is at the same time a sufficient criterion for discreteness of the spectrum (for infinite $c$ even a necessary one) (Corollary\,\ref{cor_SFA_spectrum_discrete}). This immediately confirms regular behavior of Hermite Polynomials, i.e.\ where $p_\alpha$ is the probability density function of the normal distribution.

We suggest to consider the conditions listed above as \textsl{natural conditions} of the SFA setting.
Given that in \eqref{K_def1D} $K_\alpha$ is defined as the first order dynamics $\av{\vfc{\dot{s}}{\alpha}^2}_{\vfc{\dot{s}}{\alpha} | \vfc{s}{\alpha}}$ of the observed system, a violation of the first two items can be interpreted as unphysical behavior: unbounded $K_\alpha$ means infinite velocity, unbounded $K'_\alpha$ means infinite acceleration and zero $K_\alpha$ means a frozen system state. Still, we acknowledge that there may be use cases where the observed system is not subject to such physical constraints. Then one should keep in mind that the listed ``natural conditions'' are just sufficient and not necessary.
To deal, e.g., with unbounded $K_\alpha$ or $K'_\alpha$, our results yield more fine-grained criteria which are -- however -- more complicated (compare Theorem\,\ref{thm_SFA_spectrum_discrete} to Corollary\,\ref{cor_SFA_spectrum_discrete}).

For unbounded domains, i.e.\ open-space scenarios, it is natural to assume that $p_\alpha$ vanishes at singular boundary points. By Barb\u{a}lat's Lemma, $p_\alpha$ could otherwise not be uniformly continuous which is a compelling regularity property. This marks a major structural difference between the roles of $p_\alpha$ and $K_\alpha$ in open-space scenarios: for $K_\alpha$ it is odd to vanish while for $p_\alpha$ it is odd \textsl{not} to vanish at the boundary. Clearly, that is because the integral of $p_\alpha$ is required to be finite while for $K_\alpha$ no such requirement exists. Notably, this explains why the focus in this work is mostly on dealing with $p_\alpha$ vanishing at the boundary: while vanishing $K_\alpha$ would also cause a singular setting, it is the far less critical case to consider.

All established criteria are easy to check for a given SFA problem because they only require knowledge of the coefficient functions rather than of the SFA solutions.
Our results apply beyond SFA to Sturm-Liouville problems with zero potential ($\lambda w y = (r y')'$) by setting $p_\alpha = w$ and $K_\alpha = \tfrac{r}{w}$.

\pagebreak
\bibliography{Math,MachineLearning,Miscellaneous,Predictability,ReinforcementLearning,SFA,Extern}

\pagebreak
\appendix
\section{Appendix} \label{sec:A}
\subsection{Details on SFA} \label{sec:A_sfa}
The objective of SFA is to extract from a multi-dimensional input signal $\vf{x}(t)$ a ``small'' number $m$ of signals that vary as slowly as possible over time $t$.
The original, generic SFA problem from \cite{WiskottSejnowski-2002} is as follows:
\begin{align} \label{A_sfa-F}
\text{For} \; i \in \{1, \ldots, m\}, \; \vfc{y}{i}(t) \; \coloneq \; \vfc{g}{i}(\vf{x}(t)) \notag \\
\begin{split} 
\optmin{\vfc{g}{i} \in \mathcal{F}} & \av{\dot{y}_i^2} \\
\subjectto	& \av{\vfc{y}{i}}
\hphantom{\vfc{y}{i}^2 \vfc{y}{j}} \, \!\!
= \quad 0 \quad \hphantom{\forall \; j < i} \quad \text{(zero mean)}\\
& \av{\vfc{y}{i}^2}
\hphantom{\vfc{y}{i} \vfc{y}{j}} \!\!
= \quad 1 \quad \hphantom{\forall \; j < i} \quad \text{(unit variance)}\\
& \av{\vfc{y}{i} \vfc{y}{j} }
\hphantom{ \vfc{y}{i}^2} \, \!\!
= \quad  0 \quad \forall \; j < i \quad \text{(pairwise decorrelation)}
\end{split}
\end{align}
In the original work, the problem is mainly discussed for the case that $\mathcal{F}$ is a finite-dimensional function space like polynomials up to a certain degree. Later, in \cite{SprekelerWiskott-2008}, SFA is discussed for an infinite-dimensional, ``general'' function space $\mathcal{F}$ of sufficient regularity, i.e.\ that fulfills the necessary mathematical requirements of integrability and differentiability. This forms the basis of a mathematical framework for SFA and gives rise to its extension \textsl{xSFA} in \cite{SprekelerZitoEtAl-2014}
with an application to blind source separation.
We sketch this mathematical foundation and key results of SFA theory.
For this, let ${\vf{g}(\vf{x}) = (g_1(\vf{x}), \ldots, g_m(\vf{x}))}$ denote a generic slow feature extraction, i.e.\ the $i$th extracted signal is given as ${\vfc{y}{i}(t)=\vfc{g}{i}(\vf{x}(t))}$. Note that these depend instantaneously on the input signal $\vf{x}(t)$, so SFA cannot just fulfill its goal by forming a lowpass filter.

Restricting $\mathcal{F}$ to be finite-dimensional, like in \cite{WiskottSejnowski-2002}, transforms \eqref{A_sfa-F} into an efficiently solvable eigenvalue problem. Let $\mathbf{h}$ denote a basis of $\mathcal{F}$. Then using $\mathbf{h}$ as a nonlinear expansion on the input signal $\vf{x}$, extraction can be performed by linear transformation and projection, wherein the extraction matrix is optimized over a finite training-phase $\trph$ consisting of equidistant time points. To let $\mathcal{F}$ consist of polynomials up to a certain degree, one sets $\mathbf{h}$ to consist of monomials up to that degree\footnote[1]{For higher degree, Legendre or Bernstein polynomials should be favored over monomials because of better numerical stability.}. By the Stone-Weierstrass Theorem this technique can approximate every continuous function and moreover also regulated functions (piecewise continuous). However, high degree expansion can require significant cost in terms of training data and computation. Applying SFA hierarchically like done in \cite{FranziusSprekelerEtAl-2007e, Schoenfeld2015} can help to keep these costs tractable. The initial step to solve problem \eqref{A_sfa-F} for a finite basis $\mathbf{h}$ is to sphere the expanded signal over the training-phase, i.e.\ shift its mean to zero and normalize the covariance-matrix to identity.
Let $\vf{z}$ denote the expanded representation of our input signal $\vf{x}$, sphered over a finite training phase $\trph$ with average notation ${\av{s(t)} \coloneq \frac{1}{\abs{\trph}} \sum_{t\in \trph} s(t)}$ (average over training phase):
\begin{align} \label{A_sphering1}
\tilde{\mathbf{z}}(t) \quad &\coloneq \quad \mathbf{h}(\vf{x}(t)) - \av{\mathbf{h}(\vf{x}(t))} &&\text{(make mean-free)}\\
\vf{z}(t) \quad &\coloneq \quad \mathbf{S} \tilde{\mathbf{z}}(t) \qquad \text{with} \quad \mathbf{S} \coloneq \av{\tilde{\mathbf{z}} \tilde{\mathbf{z}}^T}^{-\frac{1}{2}}, \; \vf{z}(t) \in \mathbb{R}^{n} &&\text{(normalize covariance)} \label{A_sphering2}
\end{align}
Then one can formulate the SFA problem directly in terms of $\vf{z}$:
\begin{align} \label{A_sfa}
\text{For} \; i \in \{1, \ldots, m\} \notag \\
\begin{split} 
\optmin{\mathbf{a}_i \in \mathbb{R}^{n}} & \mathbf{a}_i^T \av{\dot{\vf{z}}\dot{\vf{z}}^T} \mathbf{a}_i\\
\subjectto	& \mathbf{a}_i^T \av{\vf{z}} \hphantom{\mathbf{a}_i \mathbf{a}_j \vf{z}^T} \, = \quad 0 \quad \hphantom{\forall \; j < i} \quad \text{(zero mean)}\\
& \mathbf{a}_i^T \av{\vf{z} \vf{z}^T} \mathbf{a}_i \hphantom{\mathbf{a}_j} = \quad 1 \quad \hphantom{\forall \; j < i} \quad \text{(unit variance)}\\
& \mathbf{a}_i^T \av{\vf{z} \vf{z}^T} \mathbf{a}_j \hphantom{\mathbf{a}_i} = \quad  0 \quad \forall \; j < i \quad \text{(pairwise decorrelation)}
\end{split}
\end{align}
The sphering yields $\av{\vf{z}} = 0$ and $\av{\vf{z} \vf{z}^T} = \id$, hence the constraints simplify to
\begin{equation} \label{A_SFA-constraints-no-matrix}
\mathbf{a}_i^T \mathbf{a}_j = \delta_{ij}
\end{equation}
With ${\mathbf{A}_m \coloneq \left(\mathbf{a}_1, \ldots, \mathbf{a}_m \right) \in \mathbb{R}^{n \times m}}$, constraint \eqref{A_SFA-constraints-no-matrix} is automatically fulfilled if we choose
\begin{equation} \label{A_SFA-constraints-matrix}
\mathbf{A}_m \quad = \quad \mathbf{A}\mathbf{I}_m \quad \text{with} \quad \mathbf{A} \in \orth(n)
\end{equation}
$\orth(n) \subset \mathbb{R}^{n \times n}$ denotes the space of orthogonal transformations, i.e.\ $\mathbf{A}\mathbf{A}^T = \mathbf{I}$ and $\mathbf{I}_m \in \mathbb{R}^{n \times m}$ denotes the reduced identity matrix consisting of the first $m$ Euclidean unit vectors as columns.
Choosing $\mathbf{a}_i$ as eigenvectors of $\av{\dot{\vf{z}}\dot{\vf{z}}^T}$, corresponding to the eigenvalues in ascending order, yields $\mathbf{A}_m$ solving \eqref{A_sfa} globally and the slowest output signals are given as the $m$-dimensional signal ${\vf{y}(t)=\mathbf{A}_m^T \vf{z}(t)}$. \cite{WiskottSejnowski-2002} describes this procedure in detail.

To get an idea of what the solutions of \eqref{A_sfa} would converge to if we increase the dimension of $\mathbf{h}$, we examine what ideal SFA solutions one would expect for the unrestricted function space $\mathcal{F}$. In this work, we consider the scenario where $\vf{x}(t)$ is composed of statistically independent sources $\vfc{s}{\alpha}$. \cite{SprekelerWiskott-2008, SprekelerZitoEtAl-2014} analyze this case, proposing xSFA as an extension of SFA that can identify such sources. We recapitulate some theory and results:

Assuming that $\vf{x}(t)$ is an ergodic process, SFA can be formulated in terms of the ensemble (i.e.\ the set of possible values of $\vf{x}$ and $\dot{\vf{x}}$) using the probability density $p_{\vf{x}, \dot{\vf{x}}}(\vf{x}, \dot{\vf{x}})$. The corresponding marginal and conditional densities are defined as ${p_{\vf{x}}(\vf{x}) \coloneq \int p_{\vf{x}, \dot{\vf{x}}}(\vf{x}, \dot{\vf{x}}) d^n \dot{x}}$ and ${p_{\dot{\vf{x}} | \vf{x}}(\dot{\vf{x}} | \vf{x}) \coloneq \frac{p_{\vf{x}, \dot{\vf{x}}}(\vf{x}, \dot{\vf{x}})}{p_{\vf{x}}(\vf{x})}}$.
Further assuming that the ensemble averages ${\av{f(\vf{x}, \dot{\vf{x}})}_{\vf{x}, \dot{\vf{x}}} \coloneq \int p_{\vf{x}, \dot{\vf{x}}}(\vf{x}, \dot{\vf{x}}) f(\vf{x}, \dot{\vf{x}}) d^n x d^n \dot{x}}$,
${\av{f(\vf{x})}_{\vf{x}} \coloneq \int p_{\vf{x}}(\vf{x}) f(\vf{x}) d^n x}$ and ${\av{f(\vf{x}, \dot{\vf{x}})}_{\dot{\vf{x}} | \vf{x}}(\vf{x}) \coloneq \int p_{\dot{\vf{x}} | \vf{x}}(\dot{\vf{x}}|\vf{x}) f(\vf{x}, \dot{\vf{x}}) d^n \dot{x}}$ all exist and using the chain rule, the SFA optimization problem can be stated in terms of the ensemble:
\begin{align} \label{A_sfa-ensemble}
\text{For} \; i \in \{1, \ldots, m\} \notag \\
\begin{split} 
\optmin{\vfc{g}{i} \in \mathcal{F}} & \sum_{\gamma, \nu} \quad \av{\partial_{\gamma}\vfc{g}{i}(\vf{x}) \av{\dot{x}_\gamma \dot{x}_\nu}_{\dot{\vf{x}} | \vf{x}} \partial_{\nu}\vfc{g}{i}(\vf{x})}_{\vf{x}} \\
\subjectto	& \av{\vfc{g}{i}(\vf{x})}_{\vf{x}}
\hphantom{\vfc{g}{i}^2(\vf{x}) \vfc{g}{j}(\vf{x})} \, \!\!
= \quad 0 \quad \hphantom{\forall \; j < i} \quad \text{(zero mean)}\\
& \av{\vfc{g}{i}^2(\vf{x})}_{\vf{x}}
\hphantom{\vfc{g}{i}(\vf{x}) \vfc{g}{j}(\vf{x})} \!\!
= \quad 1 \quad \hphantom{\forall \; j < i} \quad \text{(unit variance)}\\
& \av{\vfc{g}{i}(\vf{x}) \vfc{g}{j}(\vf{x}) }_{\vf{x}}
\hphantom{ \vfc{g}{i}^2(\vf{x})} \, \!\!
= \quad  0 \quad \forall \; j < i \quad \text{(pairwise decorrelation)}
\end{split}
\end{align}
A key result from \cite{SprekelerWiskott-2008} is that given the partial differential operator
\begin{equation} \label{A_diff_op_D}
\mathcal{D} \quad \coloneq \quad -\frac{1}{p_{\vf{x}}(\vf{x})} \quad \sum_{\gamma, \nu} \; \partial_{\gamma}p_{\vf{x}}(\vf{x}) \av{\dot{x}_\gamma \dot{x}_\nu}_{\dot{\vf{x}} | \vf{x}} \partial_{\nu}
\end{equation}
the ideal solutions for SFA on an unrestricted function space can be found by solving the eigenvalue equation
\begin{equation} \label{A_dgl-x}
\mathcal{D} \vfc{g}{i}(\vf{x}) \quad = \quad \lambda_i \vfc{g}{i}(\vf{x})
\end{equation}
under the Neumann\footnote{Boundary conditions referring to the derivative value are sometimes erroneously called ``von Neumann'' boundary conditions. However, they are actually called after the mathematician Carl Gottfried Neumann rather than after John von Neumann.}
boundary conditions
\begin{equation} \label{A_neumann-x}
\sum_{\gamma, \nu} \; n_{\gamma}(\vf{x}) p_{\vf{x}}(\vf{x}) \av{\dot{x}_\gamma \dot{x}_\nu}_{\dot{\vf{x}} | \vf{x}}(\vf{x}) \partial_{\nu} \vfc{g}{i}(\vf{x}) \quad = \quad 0
\end{equation}
where $n_{\gamma}(\vf{x})$ is the $\gamma$th component of the normal vector at the boundary point $\vf{x}$. Assuming the input signal $\vf{x}(t)$ is composed of statistically independent sources $\vfc{s}{\alpha}$ for $\alpha \in \{1, \ldots, S \}$, this result can be formulated in terms of the sources. Because of statistical independence we have $p_{\vf{s}, \dot{\vf{s}}}(\vf{s}, \dot{\vf{s}}) = \prod_{\alpha}  p_{\vfc{s}{\alpha}, \vfc{\dot{s}}{\alpha}}(\vfc{s}{\alpha}, \vfc{\dot{s}}{\alpha})$,
$p_{\vf{s}}(\vf{s}) = \prod_{\alpha}  p_{\vfc{s}{\alpha}}(\vfc{s}{\alpha})$ and
$\av{\vfc{\dot{s}}{\alpha} \vfc{\dot{s}}{\beta}}_{\dot{\vf{s}} | \vf{s}}(\vf{s}) = \delta_{\alpha \beta} K_\alpha (\vfc{s}{\alpha})$.
$\mathcal{D}(\vf{s})$ can be decomposed as
\begin{equation} \label{A_dgl-x-decomp}
\mathcal{D} (\vf{s}) \quad = \quad \sum_{\alpha} \mathcal{D}_\alpha (\vfc{s}{\alpha})
\end{equation}
where $\mathcal{D}_\alpha$ is the differential operator given in \eqref{SFA_diff_op}.
Regarding this decomposition, \eqref{A_dgl-x} and \eqref{A_neumann-x} can be reformulated such that, with an additional normalization constraint, the following equations formulate SFA in terms of the sources:
\begin{align}
\mathcal{D}_{\alpha} \vfc{g}{\alpha i} \quad &= \quad \lambda_{\alpha i} \vfc{g}{\alpha i} \label{A_dgl-s}\\ 
p_{\alpha} K_\alpha  \partial_\alpha \vfc{g}{\alpha i} \quad &= \quad 0 \quad\quad\quad \text{on the boundary}\label{A_neumann-s} \\ 
\av{\vfc{g}{\alpha i}^2}_{\vfc{s}{\alpha}} \quad &= \quad 1
\end{align}
\textsl{Theorem 2} in \cite{SprekelerWiskott-2008}/\textsl{Theorem 1} in \cite{SprekelerZitoEtAl-2014} states that the solutions of \eqref{A_dgl-x} are composed of solutions of \eqref{A_dgl-s}:
\begin{align}
\vfc{g}{\mathbf{i}}(\vf{s}) \quad &= \quad \prod_{\alpha} \; \vfc{g}{\alpha \mathbf{i}_{\alpha}}(\vfc{s}{\alpha}) \label{A_SFA_product_eq} \\
\lambda_{\mathbf{i}} \quad &= \quad \sum_{\alpha} \; \lambda_{\alpha \mathbf{i}_{\alpha}}
\label{A_SFA_product_lambda_eq}
\end{align}
with $\mathbf{i} = (i_1, \ldots, i_S) \in \mathbb{N}^S$ denoting a multi index to select the right combination of sources. Choosing the $m$ smallest eigenvalues $\lambda_{\mathbf{i}}$ yields the $m$ slowest output signals.

More explicitly, a counterpart of Definition\,\ref{SL-problem} in SFA theory is formulated (for the regular case) as follows:

\begin{thm}[SFA differential equation for $\mathcal{D}_{\alpha}$] \label{A_thm_sfa_inhom_p}
(from \cite{SprekelerWiskott-2008, SprekelerZitoEtAl-2014, FranziusSprekelerEtAl-2007e})

Let $\vfc{g}{\alpha i}$ be the solution components of the SFA problem from Definition\,\ref{def_1D_sfa}. Then for some constants $\lambda_{i} > 0$, the solution obeys the differential equation
\begin{equation} \label{A_thm_sfa_inhom_p_eq}
- p_\alpha \;
\lambda_{i} \; \vfc{g}{\alpha i}
\quad = \quad
\big( p_\alpha \; K_\alpha \; \vfc{g'}{\alpha i} \big)'
\end{equation}
under the Neumann boundary condition
\begin{equation} \label{A_thm_sfa_inhom_p_eq_transversality}
g'_{\alpha i}(s_\alpha) \; K_\alpha \; p_\alpha \quad = \quad 0 \qquad \text{on the boundary}
\end{equation}
The $i$-th SFA delta value\footnote{In \cite{SprekelerWiskott-2008, SprekelerZitoEtAl-2014} and related literature these are usually denoted ``$\Delta$''. We avoid this notation because it collides with the Laplace(-Beltrami) operator which is a central term in related work. So we spell out ``delta`` as a word whenever we refer to the SFA delta values.} is given by $\lambda_i$, i.e.\
\begin{equation} \label{A_thm_sfa_inhom_p_eq_delta}
\int_{I_\alpha}  \;
{g'_{\alpha i}}^2 \; K_\alpha \; p_\alpha \; d s_\alpha \quad = \quad \lambda_i
\end{equation}
\end{thm}

We list some further results from \cite{SprekelerWiskott-2008, SprekelerZitoEtAl-2014}:

\begin{itemize}
\item The first harmonic $\vfc{g}{\alpha 1}(\vfc{s}{\alpha})$ of each source $\vfc{s}{\alpha}$ is a strictly monotonic signal of that source.

\item If $p_{\vf{s}, \dot{\vf{s}}}$ is Gaussian and isotropic in $\vf{s}$ and $\dot{\vf{s}}$, one has ${p_{s_\alpha | \dot{s}_\alpha} = p_{s_\alpha} = p_{\dot{s}_\alpha} = p_{\dot{s}_\alpha | s_\alpha}}$ and the sources are normally distributed, i.e.\ $p_{\alpha}(\vfc{s}{\alpha}) = \frac{1}{\sqrt{2 \pi}} e^{-\frac{1}{2}\vfc{s}{\alpha}^2}$. Then $\av{\vfc{\dot{s}}{\alpha}^2}_{\vfc{\dot{s}}{\alpha} | \vfc{s}{\alpha}}$ is constant and the (probabilist) Hermite polynomials $\pHermite_i$ yield the solutions $\vfc{g}{\alpha i}(\vfc{s}{\alpha}) = \frac{1}{\sqrt{2^i i!}} \pHermite_i(\frac{\vfc{s}{\alpha}}{\sqrt{2}})$ with $\lambda_{\alpha i} = \frac{i}{\av{\vfc{\dot{s}}{\alpha}^2}_{\vfc{\dot{s}}{\alpha} | \vfc{s}{\alpha}}}$.

\item If the sources are uniformly distributed, i.e.\ $p_{s_\alpha, \dot{s}_\alpha}$ is independent of $s_\alpha$, then $p_\alpha$ and $\av{\vfc{\dot{s}}{\alpha}^2}_{\vfc{\dot{s}}{\alpha} | \vfc{s}{\alpha}}$ are constant and the solutions are given by Sturm-Liouville theory as harmonic oscillations $\vfc{g}{\alpha i}(\vfc{s}{\alpha}) = \sqrt{2} \cos \big(i \pi \frac{\vfc{s}{\alpha}}{L_{\alpha}} \big)$ with ${\lambda_{\alpha i} = \av{\vfc{\dot{s}}{\alpha}^2}_{\vfc{\dot{s}}{\alpha} | \vfc{s}{\alpha}} \big(\frac{\pi}{L_{\alpha}} i \big)^2}$, assuming that $\vfc{s}{\alpha}$ takes values in the interval $[0, L_{\alpha}]$. Therefore, one refers to $\vfc{g}{\alpha i}$ as the $i$th harmonic of the source $\vfc{s}{\alpha}$.

\item The slowest signal found by SFA is plainly the first harmonic $\vfc{g}{\alpha 1}$ of the slowest source. This result is the basis of xSFA as it allows to clean subsequent signals from the first source. Iterating this procedure finally yields all sources.

\item \eqref{A_SFA_product_eq} implies that each output component $\vfc{g}{\mathbf{i}}$ of SFA is a product of harmonics $\vfc{g}{\alpha i}$ of earlier obtained sources.
\end{itemize}

In addition to these statements we point out some other well known properties that we think are of interest and were not mentioned in the work cited above:
\begin{itemize}
\item For uniformly distributed sources, all higher harmonics can be calculated from the first harmonic using the Chebyshev polynomials $\Tschebyschow_i$ as $\vfc{g}{\alpha i} = \Tschebyschow_i(\vfc{g}{\alpha 1})$.

\item By the \textsl{Sturm-Picone Comparison Theorem}, applied on the differential equation \eqref{A_dgl-s}, the zeros of $\vfc{g}{\alpha i}$ and $\vfc{g}{\alpha (i+1)}$ interlace.
\end{itemize}
\end{document}